\documentclass{article}

    \PassOptionsToPackage{numbers, compress}{natbib}

\usepackage[final]{neurips_2025}
\usepackage{subcaption}
\usepackage[export]{adjustbox}

\usepackage[utf8]{inputenc} 
\usepackage[T1]{fontenc}    
\usepackage{hyperref}       
\usepackage{url}            
\usepackage{booktabs}       
\usepackage{amsfonts}       
\usepackage{nicefrac}       
\usepackage{microtype}      
\PassOptionsToPackage{table,svgnames}{xcolor}
\RequirePackage{xcolor}

\usepackage{tcolorbox}
\tcbuselibrary{listings, skins, breakable}
\newtcolorbox{algowrapped}[1][]{%
  breakable,
  floatplacement=H,
  colback=white,
  colframe=black,
  listing only,
  listing options={style=tcblatex},
  title=Semi-Speculative Decoding for CaDDi-AR,
  #1
}

\usepackage{enumitem}
\usepackage{caption}
\usepackage{graphicx}
\usepackage{adjustbox}
\usepackage{wrapfig}
\usepackage{graphicx}
\usepackage{amsmath}
\usepackage{amssymb}
\usepackage{mathtools}
\usepackage{amsthm}
\usepackage{multirow}
\usepackage{multicol}
\usepackage{multirow}
\usepackage{array}
\usepackage{float} 
\usepackage{algorithm}
\usepackage{algorithmicx}   
\usepackage{algpseudocodex} 
\usepackage{siunitx}

\theoremstyle{plain}
\newtheorem{theorem}{Theorem}[section]
\newtheorem{proposition}[theorem]{Proposition}
\newtheorem{lemma}[theorem]{Lemma}

\theoremstyle{definition}
\newtheorem{definition}[theorem]{Definition}

\theoremstyle{remark}

\title{Non-Markovian Discrete Diffusion with Causal Language Models}

\author{
\textbf{Yangtian Zhang}\thanks{Equal contribution} \quad
\textbf{Sizhuang He}\footnotemark[1] \quad
\textbf{Daniel Levine} \quad
\textbf{Lawrence Zhao} \quad
\textbf{David Zhang} \\
\textbf{Syed Asad Rizvi} \quad
\textbf{Shiyang Zhang} \quad
\textbf{Emanuele Zappala} \quad
\textbf{Rex Ying}\thanks{Correspondence to: \texttt{rex.ying@yale.edu}, \texttt{david.vandijk@yale.edu}} \quad
\textbf{David van Dijk}\footnotemark[2] \\
Yale University, New Haven, CT, USA \\
\texttt{\{yangtian.zhang, sizhuang.he, daniel.levine, lawrence.zhao\}@yale.edu} \\
\texttt{\{david.zhang, syed.rizvi, shiyang.zhang\}@yale.edu} \quad 
\texttt{emanuele.zappala@isu.edu} \\
\texttt{\{rex.ying, david.vandijk\}@yale.edu}
}

\newcommand{\method}{CaDDi}

\DeclareMathOperator{\E}{\mathbb{E}}

\begin{document}

\maketitle

\begin{abstract}
Discrete diffusion models offer a flexible, controllable approach to structured sequence generation, yet they still lag behind causal language models in expressive power. A key limitation lies in their reliance on the Markovian assumption, which restricts each step to condition only on the current state, leading to potential uncorrectable error accumulation. In this paper, we introduce \textbf{CaDDi} (\underline{Ca}usal \underline{D}iscrete \underline{Di}ffusion Model), a discrete diffusion model that conditions on the entire generative trajectory, thereby lifting the Markov constraint and allowing the model to revisit and improve past states. By unifying sequential (causal) and temporal (diffusion) reasoning in a single non-Markovian transformer, CaDDi also treats standard causal language models as a special case and permits the direct reuse of pretrained LLM weights with no architectural changes. Empirically, CaDDi outperforms state-of-the-art discrete diffusion baselines on natural-language benchmarks, substantially narrowing the remaining gap to large autoregressive transformers.
\end{abstract}

\section{Introduction}
\vspace{-5pt}
Autoregressive transformers have become a dominant approach for sequence modeling~\citep{vaswani2017attention, chowdhery2023palm, touvron2023llama}, achieving state-of-the-art performance in many natural language and biological tasks. Their left-to-right decoding paradigm simplifies training via next-token prediction and is supported by large-scale pretraining, unlocking broad linguistic (or domain) knowledge. However, these models can be less flexible for {bidirectional or partially specified} generation, such as text infilling or prompting from arbitrary locations.

By contrast, discrete diffusion models~\citep{dieleman2022continuous, d3pm, gulrajani2024likelihood, gat2024discreteflowmatching} naturally accommodate {controllable} generation scenarios where tokens can be iteratively refined and sampled in a bidirectional manner~\citep{shen2023film}. Recent advances have extended discrete diffusion to continuous time~\citep{campbell2022continuous, shi2024simplified}, improved its training objectives~\citep{shi2024simplified, lou2024discretediffusionmodelingestimating, mdlm, udlm}, and accelerated inference~\citep{park2024textit, liu2024discrete}. Yet, these models often lag behind autoregressive approaches in generation quality~\citep{zheng2024masked}, due in part to their reliance on a \textit{single} latent state for denoising—leading to fragile inference where small decoding errors can accumulate over time~\citep{xu2024energy, zheng2024masked, hu2024mask}.

To address this limitation, in this work, we explore a non-Markovian diffusion framework to discrete sequences by allowing each denoising step to condition on the entire generative trajectory, rather than a single latent state. While prior work such as DART~\citep{DART} has investigated non-Markovian diffusion for continuous data, its formulation relies on continuous kernels and smooth transitions that do not directly translate to the discrete setting. Notably, when instantiated for discrete sequences, non-Markovian diffusion naturally mirrors the structure of token-level autoregressive models—highlighting a close connection between the two.

Leveraging this insight, we propose \method{}, a causal discrete diffusion model that unifies \textit{sequential} (left-to-right) and \textit{temporal} (multi-step) dimensions in a single decoder-only transformer architecture. As a result, \method{} can be trained efficiently via a simple next-block/next-token prediction loss—similar to a causal language model—while preserving the bidirectional control and iterative refinement of diffusion. 

Additionally, we show that its token-level variation \method{}-AR can be viewed as a generalization of traditional autoregressive models ($T=1$ is the special case), making it straightforward to fine-tune a pretrained LLM for discrete diffusion. Such adaptation unlocks flexible generation modes (e.g., text infilling) without sacrificing the rich knowledge encoded by large-scale pretraining.

In summary, our key contributions are as follows.
\begin{itemize}[leftmargin=15pt, labelsep=*, itemsep=1ex]
    \item We introduce a non-Markovian discrete diffusion framework where each denoising step incorporates the full generative trajectory, improving inference robustness.

    \item We propose \method{}, a causal discrete diffusion model that unifies sequential and temporal modeling within a non-Markovian diffusion framework. Its further variation \method{}-AR generalizes traditional causal language models as a special case and can seamlessly adopt pretrained LLMs for discrete diffusion, enabling more controllable and structured generation.

    
    \item Quantitative results show that CaDDi outperforms recent discrete diffusion models, achieving lower generative perplexity on language datasets and stronger reasoning capabilities when leveraging a pretrained LLM.
    
\end{itemize}

\section{General Discrete Diffusion}

Recently, discrete diffusion models~\citep{d3pm} have emerged as a powerful framework. In contrast to their continuous counterparts, which corrupt data by adding Gaussian noise in a real-valued space, discrete diffusion models operate on categorical variables, gradually corrupting tokens before reconstructing them through a learned denoising process.

\paragraph{Forward Process.}
Let \(\mathbf{x}_0 = (\mathbf{x}_0^1, \mathbf{x}_0^2, \ldots, \mathbf{x}_0^L)\) be a sequence of discrete tokens from a vocabulary \(\mathcal{V}\) of size \(|\mathcal{V}|\). The forward (noising) process produces latent variables \(\mathbf{x}_1, \mathbf{x}_2, \ldots, \mathbf{x}_T\) by iteratively corrupting each token according to a time-dependent transition matrix \(\mathbf{Q}_t\):
\begin{align}
    q\bigl(\mathbf{x}_t \mid \mathbf{x}_{t-1}\bigr)
    &= \operatorname{Cat}\Bigl(\mathbf{x}_{t};\,  \mathbf{x}_{t-1}\,\mathbf{Q}_t\Bigr),
\end{align}
where \(\operatorname{Cat}(\cdot;\,\pi)\) denotes the categorical distribution with parameter \(\pi\), and \([\mathbf{Q}_t]_{ij}\) gives the probability of transitioning from state \(i\) to state \(j\) at time \(t\). Due to the Markovian assumption, the marginal distribution at any timestep  can be computed in closed form:
\begin{equation}
q(\mathbf{x}_t \mid \mathbf{x}_0) = \operatorname{Cat}(\mathbf{x}_t; \mathbf{x}_0 \bar{\mathbf{Q}}_t), \quad \text{where} \quad \bar{\mathbf{Q}}_t = \mathbf{Q}_1\mathbf{Q}_2\dots\mathbf{Q}_t.
\end{equation}

Common choices for the transition kernel $\mathbf{Q}_t$ include uniform kernel \(\mathbf{Q_t} = (1-\beta_t) \mathbf{I} + \beta_t \mathbf{1}\mathbf{1}^T/|\mathcal{V}|\) of size \(|\mathcal{V}| \times |\mathcal{V}|\)
and absorbing kernel $\mathbf{Q}_t=\left(1-\beta_t\right) \mathbf{I}+\beta_t \mathbf{1} e_m^{\top}$ of size \((|\mathcal{V}| +1) \times (|\mathcal{V}| + 1)\), where $\mathbf{1}$ is an all-one vector and $e_m^{\top}$ represents a one-hot vector where element at index $m=|\mathcal{V}|+1$ is 1.

\paragraph{Reverse Process.}
The reverse (denoising) model \(p_\theta\bigl(\mathbf{x}_{t-1} \mid \mathbf{x}_t\bigr)\) is commonly parameterized in form of posterior distribution , which is learned to reverse the corruption process
\begin{equation}
\label{eq:markov_posterior}
    p_\theta\bigl(\mathbf{x}_{t-1} \mid \mathbf{x}_t\bigr)= 
    q\bigl(\mathbf{x}_{t-1} \mid \mathbf{x}_t, \mathbf{x}_0=\boldsymbol{\mu}_\theta\left(\mathbf{x}_t, t\bigr)\right)=
    \operatorname{Cat}\left(\mathbf{x}_{t-1} ; \frac{\mathbf{x}_t \mathbf{Q}_t^{\top} \odot \boldsymbol{\mu}_\theta\left(\mathbf{x}_t, t\right) \bar{\mathbf{Q}}_{t-1}}{\boldsymbol{\mu}_\theta\left(\mathbf{x}_t, t\right) \bar{\mathbf{Q}}_t \mathbf{x}_t^{\top}}\right)
\end{equation}
where \(\boldsymbol{\mu}_\theta\left(\mathbf{x}_t, t\right) = p_\theta (\mathbf{x}_0 \mid \mathbf{x}_t)\) is a time-dependent denoiser (\textit{x\textsubscript{0}-parameterization}) mapping \(\mathbf{x}_t\) back to the mean distribution of clean data.

\paragraph{Error During Inference.}
During inference, discrete diffusion models can potentially encounter a variety of sources of error, including inaccuracies arising from heavy-tailed per-token sampling distributions, and independent factorization in denoising distribution $p_\theta\left(\mathbf{x}_0 \mid \mathbf{x}_t\right)=\Pi_i p_\theta\left(\mathbf{x}_0^i \mid \mathbf{x}_t\right)=\Pi_i \boldsymbol{\mu}_\theta^i\left(\mathbf{x}_t, t\right)$, which prevents the model from fully capturing the true posterior $q\left(\mathbf{x}_0 \mid \mathbf{x}_t\right)$ that inherently exhibits dependencies across token dimensions~\citep{xu2024energy}.

Although each individual source of error may be moderate, they can propagate and accumulate across reverse steps. Specifically, incorrectly predicted $\mathbf{x}_t$ will not only impact the immediate denoiser output $\boldsymbol{\mu}_\theta\left(\mathbf{x}_t, t\right)$, but also affect subsequent predictions, leading to further deviations as the trajectory progresses. Even when denoiser predictions are accurate, the constrained form of the approximated posterior can limit the model's ability to self-correct earlier mistakes.\footnote{One illustrative case of such limited self-correction is the \textit{failure-to-remask} phenomenon under absorbing diffusion kernels, See Appendix}

\begin{figure*}[t]
    \centering
    \includegraphics[width=\linewidth]{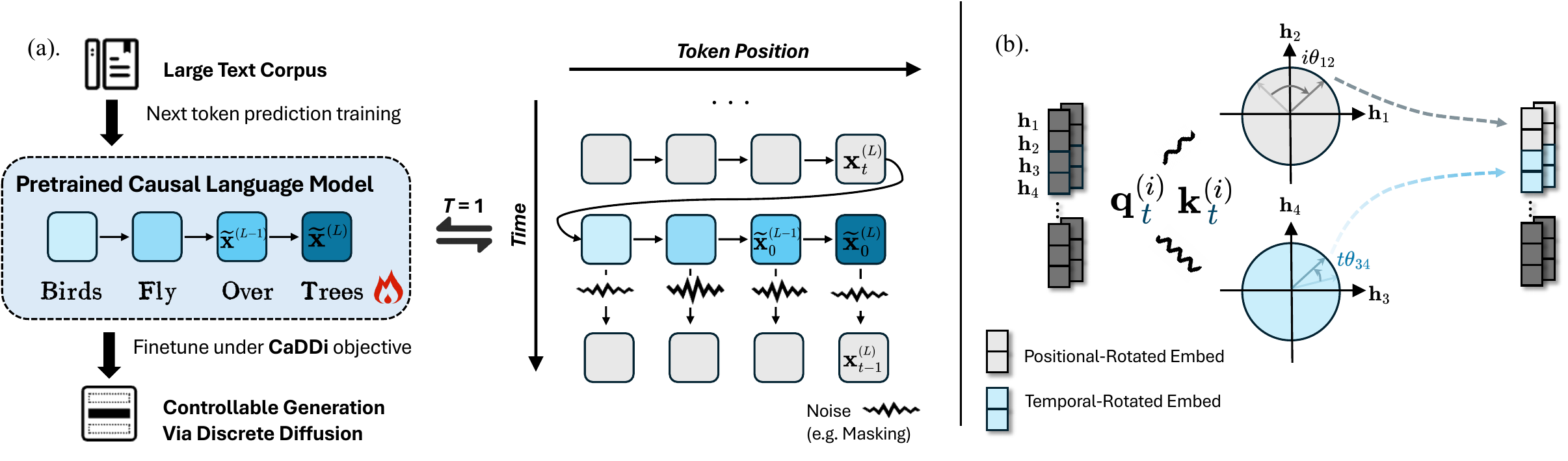}
    \caption{(a). \textbf{Inference paradigm for a standard causal language model versus \method{}-AR.} In \method{}-AR, each timestep first autoregressively denoises the tokens into $\widetilde{\mathbf{x}}_0$, then re-applies noise via the diffusion kernel to obtain ${\mathbf{x}}_{t-1}$. A traditional autoregressive model emerges as the special case of $T=1$, which can be adapted to discrete diffusion by fine-tuning. (b). \textbf{Extending 1D to 2D Rotary Positional Encoding.} Standard rotary encodings for token positions are seamlessly generalized to also encode diffusion timesteps, remaining fully backward-compatible with existing language model architectures.}
    \label{fig:model}
\end{figure*}

\section{Non-Markovian Discrete Diffusion}
\label{sec:framework}

\subsection{Is the Markovian Assumption Necessary?}
Diffusion models can be viewed as a special class of hierarchical variational autoencoders (HVAEs) with a Markovian assumption, where the latent variables consist of progressively corrupted versions of the original data. Let \(\mathbf{x}_0\) be the original data and \(\mathbf{x}_{1:T}\) be the latent variables at timesteps \(1,\dots,T\). The ELBO objective of HVAE can be written as:
\begin{align}
    \max_{\theta, \phi} \;\; \mathcal{L}_{\theta, \phi}^{\mathrm{ELBO}} 
    &= \mathbb{E}_{\mathbf{x}_{1: T} \sim q_\phi(\mathbf{x}_{1: T} \mid \mathbf{x}_0)} \Bigg[ 
        \sum_{t=1}^T \log p_\theta\bigl(\mathbf{x}_{t-1} \mid \mathbf{x}_{t:T}\bigr) 
       + \log p_\theta\bigl(\mathbf{x}_T\bigr)
         \;-\; \log q_\phi\bigl(\mathbf{x}_{1:T} \mid \mathbf{x}_0\bigr) 
    \Bigg],
\end{align}
where \(q_\phi(\mathbf{x}_{1:T} \mid \mathbf{x}_0)\) is the variational posterior distribution, and \(p_\theta(\mathbf{x}_{t-1} \mid \mathbf{x}_{t:T})\) is the generative (reverse) model.
Most of the existing discrete diffusion models make Markovian assumptions by specifying a fixed, non-learnable forward process \(q_\phi\bigl(\mathbf{x}_{1:T} \mid \mathbf{x}_0\bigr)=\prod_{t=1}^T q\bigl(\mathbf{x}_{t} \mid \mathbf{x}_{t-1}\bigr)\), and modeling the reverse process as a time-localized transition \(p_\theta\bigl(\mathbf{x}_{t-1} \mid \mathbf{x}_{t:T}\bigr)\!=\!p_\theta\bigl(\mathbf{x}_{t-1} \mid \mathbf{x}_{t}\bigr)\), which is approximated by a posterior estimator such as \(q\bigl(\mathbf{x}_{t-1} \mid \mathbf{x}_{t}, \mathbf{x}_0 = \mu_\theta(\mathbf{x}_t, t)\bigr)\). While beneficial for computational efficiency, the Markovian assumption compresses all relevant information into a single state $\mathbf{x}_t$, making the model more susceptible to error accumulation through the generative trajectory, which is shown in prior work~\citep{xu2024energy} and our experiments in Figure~\ref{fig:robust}.

In this work, we explore a more flexible alternative: treating discrete diffusion models as general HVAEs \textit{without requiring the Markovian assumption}. Inspired by recent studies~\citep{DART}, we consider a non-Markovian generative process, where the reverse model $p_\theta\left(\mathbf{x}_{t-1} \mid \mathbf{x}_{t: T}\right)$ can access the entire future trajectory $\mathbf{x}_{t: T}$ to denoise $\mathbf{x}_{t-1}$ for denoising in an autoregressive manner. Assuming an independently parameterized forward corruption process, the reverse model simplifies to: $p_\theta\left(\mathbf{x}_{t-1} \mid \mathbf{x}_{t: T}\right)=q\left(\mathbf{x}_{t-1} \mid \mathbf{x}_0 = \mu_\theta\left(\mathbf{x}_{t: T}, t\right)\right)$, relaxing the structural constraints imposed by the Markovian factorization in both the forward and reverse directions.

In the following, we will describe how the non-Markovian discrete diffusion process is constructed. Crucially, we will see that the resulting non-Markovian autoregressive inference mechanism essentially aligns with a causal language model plus an additional temporal dimension—laying the groundwork for our unified spatial-temporal framework (Section~\ref{sec:method}).

\subsection{Non-Markovian Forward Process}
Following ~\citep{DART}, we adopt a simple yet expressive \textit{independent} noising process to ensure that the latent trajectory carry more complementary information across timesteps. Instead of relying on a stepwise Markov chain, we inject independent noise into the original data \(\mathbf{x}_0\) at each timestep, as illustrated in Fig.~\ref{fig:nonmarkov-alg-fig}. Formally, the forward process is defined as:
    \begin{align}
    q(\mathbf{x}_{0:T}) :=  q(\mathbf{x}_0, \mathbf{x}_1, \mathbf{x}_2, ..., \mathbf{x}_T) 
    &= q(\mathbf{x}_0) \prod_{t=1}^{T} q(\mathbf{x}_t | \mathbf{x}_{0:t-1}) 
    = q(\mathbf{x}_0) \prod_{t=1}^{T} q(\mathbf{x}_t | \mathbf{x}_{0}),
    \end{align} 
where the final equality follows from our assumption that noise is added independently at each timestep, conditioned only on $\mathbf{x}_0$. While the per-step marginals $q\left(\mathbf{x}_t \mid \mathbf{x}_{0: t-1}\right)=q\left(\mathbf{x}_t \mid \mathbf{x}_0\right)$ appear similar to those in Markovian diffusion models, the conditional independence of $\mathbf{x}_t$ from $\mathbf{x}_{t-1}$ (given $\mathbf{x}_0$ ) introduces a fundamentally different structure into the forward trajectory.

In this non-Markovian formulation, the forward process depends solely on the marginal corruption kernel $q\left(\mathbf{x}_t \mid \mathbf{x}_0\right)$, rather than the usual Markovian kernel $q\left(\mathbf{x}_t \mid \mathbf{x}_{t-1}\right)$. In practice, one can still leverage standard discrete diffusion kernels-such as the absorbing or uniform kernels-or even mixtures thereof, without enforcing a temporal Markov constraint. Additionally, we can establish the following correspondence:

\begin{proposition}
\label{prop:correspondence}
    An absorbing-state non-Markovian discrete diffusion process with marginal transition kernel \(\bar{\mathbf{Q}}_t = \left(1-\alpha_t\right) \mathbf{I}+\alpha_t \mathbf{1} e_m^{\top}\) admits a bijection to an absorbing-state Markovian discrete diffusion process with marginal transition kernel \(\bar{\mathbf{Q}}_t^* = \left(1-\alpha_t^*\right) \mathbf{I}+\alpha_t^* \mathbf{1} e_m^{\top}\) such that the two processes exhibit identical mutual information decay between $x_{t:T}$ and $x_0$, provided that the coefficients satisfy \(\alpha_t^* = \prod_{\tau =t }^T \alpha_\tau\). Proofs is provided in Appendix~\ref{app:NM_vs_Markov}
\end{proposition}

\begin{figure*}[t]
\centering

\begin{minipage}{0.52\textwidth}
    \centering
    \includegraphics[width=1.05\linewidth]{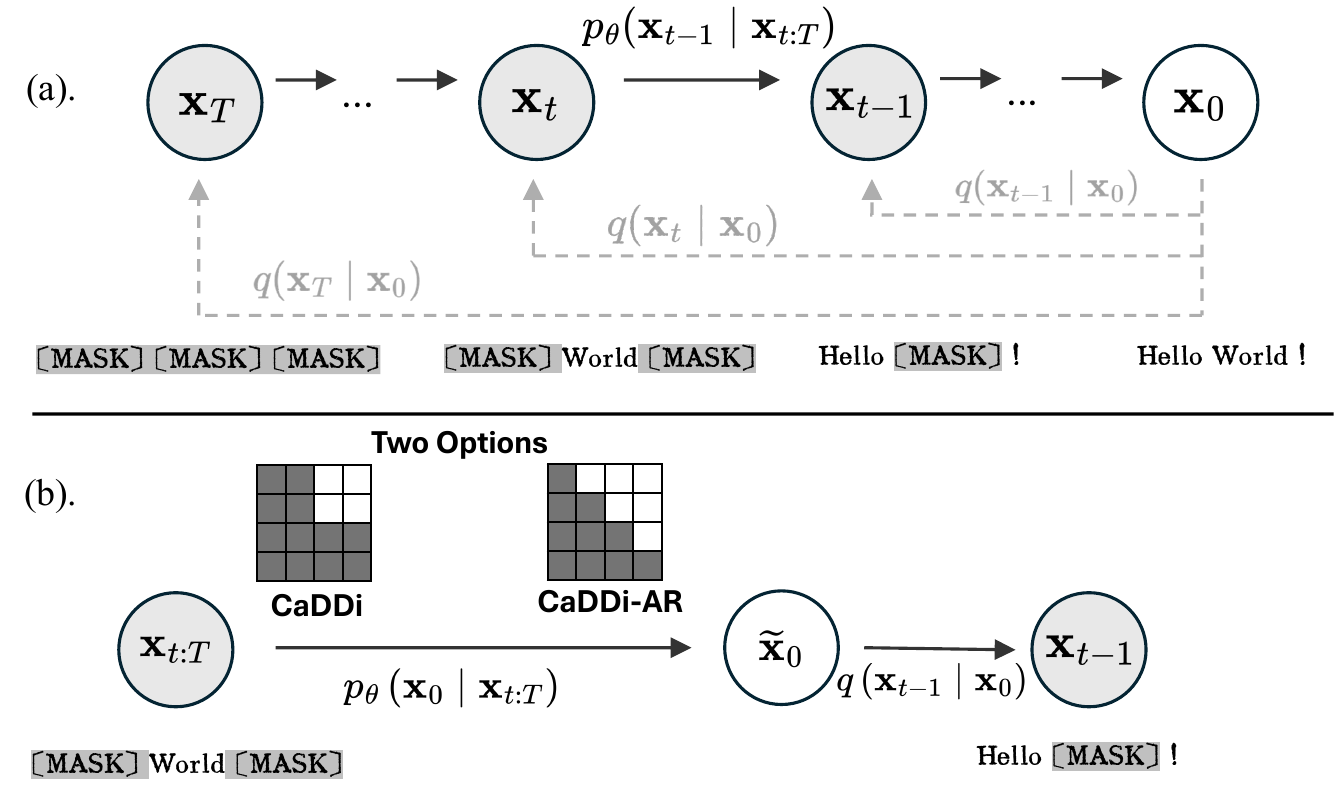}
\end{minipage}
\hfill
\begin{minipage}{0.44\textwidth}
    \centering
    \hspace*{-3.5em}  %
    \scalebox{0.85}{  
    \begin{minipage}{\textwidth}
    \begin{algorithm}[H]
    \caption{Inference for Non-Markovian Discrete Diffusion}
    \label{alg:non_markov_inference}
    \begin{algorithmic}[1]
    \State \textbf{Input:} Prior distribution $q(\mathbf{x}_T)$, model parameters $\theta$
    \State \textbf{Output:} Sampled data $\mathbf{x}_0$
    \State Initialize $\mathbf{x}_T \sim q(\mathbf{x}_T)$ as the noisy input at the final timestep
    \For{$t = T$ down to $1$}
        \State Predict $\widetilde{\mathbf{x}}_{0} \sim \mu_\theta\left( \mathbf{x}_{0} \mid \mathbf{x}_{t:T}, t \right)$
        \State Sample $\mathbf{x}_{t-1} \sim q\left( \mathbf{x}_{t-1} \mid \widetilde{\mathbf{x}}_{0} \right)$
    \EndFor
    \State \Return $\mathbf{x}_0$
    \end{algorithmic}
    \end{algorithm}
    \end{minipage}
    }
\end{minipage}

\vspace{0.5em}

\caption{
Illustration of the Non-Markovian discrete diffusion inference algorithm (right) and its forward/reverse process visualization (left).
}
\label{fig:nonmarkov-alg-fig}

\end{figure*}

\subsection{Non-Markovian Inference Process}
Similar to conventional diffusion models, our non-Markovian discrete diffusion model employs a posterior-inspired parameterization for the reverse process $p_\theta\left(\mathbf{x}_{t-1} \mid \mathbf{x}_t\right)$ However, due to the assumption of an independent corruption kernel—where the forward noise depends only on \(\mathbf{x}_0\)—we can further simplify the posterior form: 
\begin{equation}
    p_\theta(\mathbf{x}_{t-1} \mid \mathbf{x}_{t:T}) := q\bigl(\mathbf{x}_{t-1} \mid \mathbf{x}_{t:T}, \mathbf{x}_0=\boldsymbol{\mu}_\theta\left(\mathbf{x}_{t:T}, t\right)\bigr) = q\bigl(\mathbf{x}_{t-1} \mid  \mathbf{x}_0=\boldsymbol{\mu}_\theta\left(\mathbf{x}_{t:T}, t\right)\bigr)
\end{equation}
Notice that it lift the constraint introduced by \(\mathbf{x}_{t:T}\) in the posterior form and delegating all the flexibility for the denoiser model \(\mu_\theta\). As illustrated in Fig.~\ref{fig:nonmarkov-alg-fig} (right), inference proceeds in an autoregressive manner: at each timestep, we first predict a clean estimate $\widetilde{\mathbf{x}}_0 \sim \mu_\theta\left(\mathbf{x}_{t: T}, t\right)$, then sample $\mathbf{x}_{t-1}$ using the forward corruption kernel $q\left(\mathbf{x}_{t-1} \mid \widetilde{\mathbf{x}}_0\right)$. This loop continues backward through time until the final data sample $\mathbf{x}_0$ is recovered.

\subsection{Evidence Lower Bound}
As with traditional discrete diffusion models, the non-Markovian variant can be trained using a variational objective. Specifically, we optimize the Evidence Lower Bound (ELBO) as:
\begin{equation}
    \mathcal{L}_{\text{non-markov}} = 
    \mathbb{E}_{q (\mathbf{x}_{1:T} \mid \mathbf{x}_0)}
    \
    \log p_\theta(\mathbf{x}_0 \mid \mathbf{x}_{1:T})
    -\operatorname{KL}\bigl(q(\mathbf{x}_T \mid \mathbf{x}_0) \| p_\theta (\mathbf{x}_T)\bigr)
     - \mathcal{L}_T
    \label{eq:elbo_loss}
\end{equation}
 where \(\mathcal{L}_T = \sum_{t=2}^T 
    \mathbb{E}_{q( \mathbf{x}_{t:T} \mid \mathbf{x}_0 )}
    \operatorname{KL}(q (\mathbf{x}_{t-1} \mid \mathbf{x}_0) \| p_\theta\bigl(\mathbf{x}_{t-1} \mid \mathbf{x}_{t:T} )) \).
 The second term is constant for most forward kernels. For specific kernel choices (e.g. absorbing or uniform), the ELBO simplifies further:
\begin{proposition}
\label{prop:objective}
Suppose the non-Markovian diffusion process adopt an absorbing marginal kernel \(q\left(\boldsymbol{x}_t \mid \boldsymbol{x}_0\right) = \operatorname{Cat}\left(\mathbf{x}_t ; \mathbf{x}_0 \bar{\mathbf{Q}}_t\right)\), where \(\bar{\mathbf{Q}}_t = \left(1-\alpha_t\right) \mathbf{I}+\alpha_t \mathbf{1} e_m^{\top}\) and \(\alpha_t\) is a increasing function with \(\alpha_0 \approx 0\) and \(\alpha_T \approx 1\). The ELBO loss in Equation~\eqref{eq:elbo_loss} can be further simplified to (see Appendix~\ref{sec:proof} for derivation):
\begin{equation}
    \mathcal{L}_{\text{absorb}} = 
    \mathbb{E}_{\mathbf{x}_{1: T} \sim q\bigl( \mathbf{x}_{1:T} \mid \mathbf{x}_0\bigr)}
    \sum_{t=1}^T
    \bigl[
    \alpha_{t-1} \mathbf{x}_0^{\top}\log \mu_\theta(\mathbf{x}_{t:T}, t)
    \bigr].
    \label{eq:elbo_loss_absorb}
\end{equation}
\end{proposition}
The above proposition shows that for absorbing-type forward kernels.  training reduces to a weighted cross-entropy between the model’s prediction and the ground-truth clean input \(\mathbf{x}_0\), with time-dependent weights \(\alpha_{t-1}\) reflecting the degree of corruption.


\vspace{-5pt}
\section{\method{}: \underline{Ca}usal \underline{D}iscrete \underline{Di}ffusion Model}
\label{sec:method}
In natural language and image modeling, decoder-only causal models have demonstrated strong performance in autoregressive modeling, due to their efficient parallel training and scalability~\citep{radford2019language, child2019generatinglongsequencessparse, chen2020generative, ramesh2021zero}. Meanwhile, the reverse process in non-Markovian diffusion models is inherently autoregressive: we decode the entire sequence of latent states \(\mathbf{x}_{t:T}\) to retrieve \(\mathbf{x}_{t-1}\). Building on this observation, we introduce \textbf{\method{}}, a causal discrete diffusion model that unifies the \textit{sequential} dimension (i.e., token order) and the \textit{temporal} dimension (i.e., discrete diffusion timesteps) within a single Transformer. Specifically, \method{} employs a standard left-to-right architecture while conditioning on multiple timesteps from the diffusion chain, enabling it to model non-Markovian discrete diffusion in a unified framework.

\subsection{Unified Sequential and Temporal Modeling}
To accommodate both sequential and temporal dependencies in a single model, a straightforward choice is to construct a data instance as a non-Markovian forward trajectory:
\begin{equation*}
    \bigl(\mathbf{x}_T^{(0)}, \ldots, \mathbf{x}_T^{(L)},\; \mathbf{x}_{T-1}^{(0)}, \ldots, \mathbf{x}_{0}^{(L)}\bigr),
\end{equation*}
where the upper index \((i)\) denotes the token position in the original sequence of length \(L\), and the subscript denotes the diffusion timestep. As illutrated in Fig.~\ref{fig:nonmarkov-alg-fig}, block-wise causal mask is then applied, where we can use the logits from the last block (i.e. block corresponding to the position of \(\mathbf{x}_t\)) as the prediction \(\mathbf{x}_0\) of the denoiser, assuming a \(\mathbf{x}_0\)-parameterization. In inference, we do the \textit{block-wise autoregression} where predicted clean data \(\widetilde{\mathbf{x}}_0\) is used for constructing the next latent variable \(\mathbf{x}_{t-1}\).

\paragraph{2D Rotary Positional Encoding.} 
\label{subsec:2d_rotary}
While it's straightforward to use a decoder-only causal model for modeling the sequential and temporal dependency of latent chains. One of the issue is that modern causal language model 
typically encodes \textit{only} the sequential dimension, via rotary positional encodings~\citep{roformer}. However, for non-Markovian discrete diffusion, we must capture not just the standard token-level sequence but also a temporal dimension corresponding to diffusion timesteps. To address this, we extend the original 1D rotary scheme to a {2D} variant. 

Specifically, standard RoPE in modern language models~\citep{biderman2023pythia} rotates a subset
of the query/key dimensions according to the token position $i$. If $\mathbf{R}^{(i)}$ denotes the rotation matrix parameterized by $i$, the attention weight between positions $i$ and $j$ becomes $\left(\mathbf{R}^{(i)} \mathbf{q}^{(i)}\right)^{\top}\left(\mathbf{R}^{(j)} \mathbf{k}^{(j)}\right)$, where $\mathbf{q}^{(i)}, \mathbf{k}^{(j)}$ are the query/key vectors, and $\mathbf{R}^{(i)}$ applies 2D rotations defined by position index \(i\) to corresponding pairs of dimensions. To incorporate the timestep \(t\), we introduce an additional rotation along a disjoint subspace of the embedding. This results in a block-diagonal like rotation matrix: 
\begin{equation}
\mathbf{R}_t^{(i)}=\left[\begin{array}{cc}
\mathbf{R}_{\mathrm{seq}}^{(i)} & 0 \\
0 & \mathbf{R}_{\mathrm{time}}^{(t)}
\end{array}\right]
\end{equation}
where $\mathbf{R}_{\text {seq }}^{(i)}$ is the valid part of position-based rotation matrix as before, and $\mathbf{R}_{\text {time }}^{(t)}$ applies the same rotational principle to a separate set of dimensions using the timestep $t$. By interleaving temporal based rotation in these additional dimensions, it's easy to observe that when two tokens share the same timepoints $t$:
\begin{equation*}
\begin{aligned}
\left(\mathbf{R}_t^{(i)} \mathbf{q}_t^{(i)}\right)^{\top}\left(\mathbf{R}_t^{(j)} \mathbf{k}_t^{(j)}\right) ={\mathbf{q}_t^{(i)}}^{\top} \mathbf{R}_0^{(j-i)} \mathbf{k}_t^{(j)} 
=\left(\mathbf{R}^{(i)} \mathbf{q}_t^{(i)}\right)^{\top}\left(\mathbf{R}^{(j)} \mathbf{k}_t^{(j)}\right)
\end{aligned}
\end{equation*}
which means that in the same timepoint the sequential attention pattern is identical to that of a conventional causal model, and \(\mathbf{R}_t^{(i)}\) reduces to the usual (1D) rotation in \(i\). 

In practice, feeding all timesteps into the model can be prohibitively large. Thus we have proposed several optional strategies such as latent truncation and trajectory re-composition to compress latent \(\mathbf{x}_{t:T}\) into fixed context window size. Details can be found in Appendix~\ref{sec:implement}

\subsection{CaDDi-AR: Factorization over Token Space}
\label{sec:caddi_ar}
As dicussed in~\citep{DART, xu2024energy, xiao2021tackling}, independent factorization over dimensions in one reverse step can make $p_\theta\left(\mathbf{x}_{t-1} \mid \mathbf{x}_{t: T}\right)$ neglecting the token dependence and fail to match the true posterior $q\left(\mathbf{x}_{t-1} \mid \mathbf{x}_{t: T}\right)$. A straightforward solution following~\citep{DART} is to further factorizing over token space - leading to a token-level autoregression $p_\theta\left(\mathbf{x}_{t-1} \mid \mathbf{x}_{t: T}\right)=\prod_{i=0}^L p_\theta\left(\mathbf{x}_{t-1}^i \mid \mathbf{x}_{t-1}^{0:i-1}, \mathbf{x}_{t: T}\right)$ and next-token prediction loss. We denote this varient as \method{}-AR. Notice it naturally fits our decoder-only causal language model, where the sampling process essentially correpond to autoregressive generation with historical trajectory \(\mathbf{x}_{t:T}\) serving as "prompt". The token-level autoregression decomposing enables \method{}-AR a more fine-grained granularity in per-step generation. 

\paragraph{Semi-Speculative Decoding}
\label{sec:semi-sepculative}
\begin{wrapfigure}[20]{r}{0.50\columnwidth}
  \vspace{-1.0\baselineskip}
  \centering
  \includegraphics[width=0.46\columnwidth]{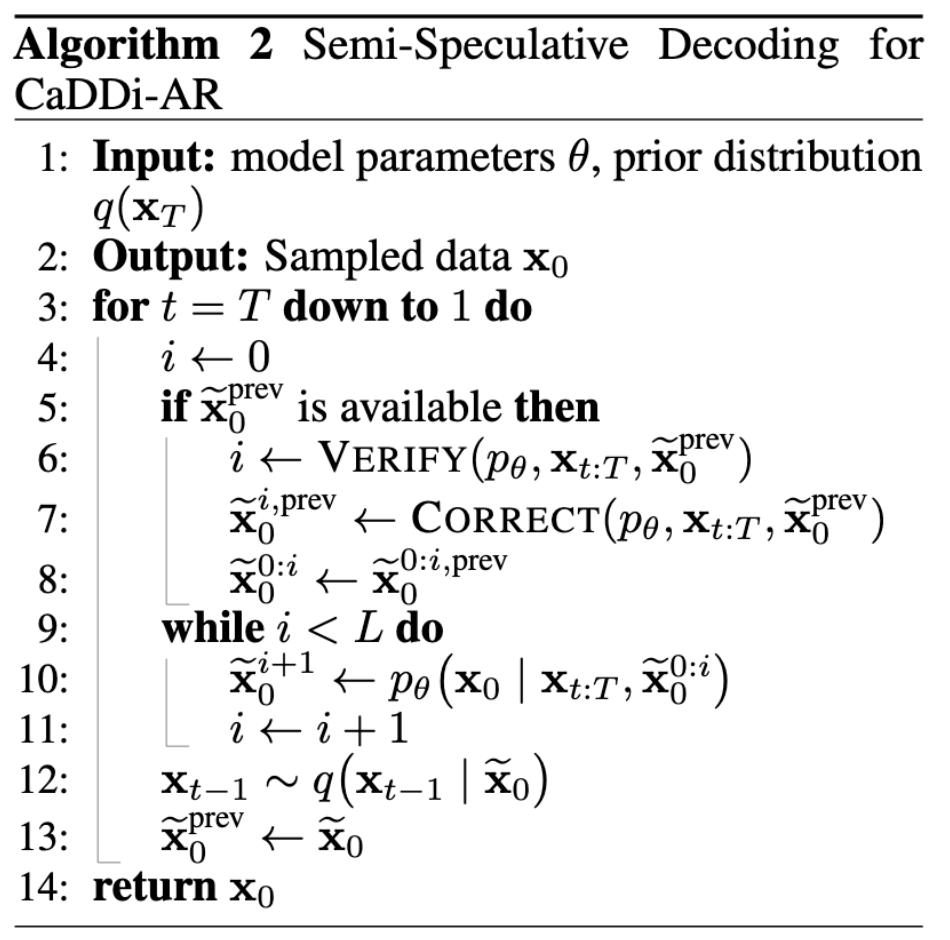}  
  \caption{Semi-Speculative Decoding with CaDDi-AR: The model verifies all tokens in parallel to identify the first rejection index \(i\), then resumes sampling from that point.}
  \vspace{-1.0\baselineskip}
  \label{alg:semi_spec_decoding}
\end{wrapfigure}

\method{}-AR usually gives stronger performance compared with vanilla \method{}. However, naive generation can be much slower as it requires $\mathcal{O}(L \times T)$ function evaluations for a sequence of length $L$ over $T$ timesteps. By leveraging the unique properties of decoder-only causal language modeling, we propose a \textit{semi-speculative decoding} strategy that substantially reduces inference time while maintaining generation quality.

Specifically, although causal language models generate tokens sequentially, they can verify the probabilities of any \textit{pre-drafted} sequence in parallel. Since \method{}-AR shares the same denoising target \(\mathbf{x}_0\) across all timesteps, this suggests a natural procedure: reuse the previous timestep's predictions \(\widetilde{\mathbf{x}}_0^{\text{prev}}\) as a \textit{draft} for the current timestep (see Algorithm~2). The model then \textit{verifies} these drafted tokens in parallel, accepting those that meet a specified confidence threshold (e.g., high probability).

This approach closely resembles speculative decoding~\citep{leviathan2023fast, chen2023accelerating}, with one key difference: we do not rely on a separate, smaller model to propose the draft sequence. Instead, \method{}-AR’s own predictions from the preceding timestep serve as the draft. Like speculative decoding, various verification and correction strategies (e.g., greedy, nucleus sampling) can be employed, ensuring either a comparable or identical sampling distribution while significantly reducing the total number of sampling steps.

\subsection{Leveraging Pretrained LLM for Discrete Diffusion}
\label{sec:finetune}
One key observation is that standard causal language modeling can be seen as a \textbf{special case} of CaDDi-AR under particular settings:  specifically, when using a single-step diffusion process (\(T=1\)) without any latent trajectory conditioning as 'prompt', i.e., \(p_\theta\left(\mathbf{x}_{0}^i \mid \mathbf{x}_{0}^{0:i-1}, \varnothing\right)\). Furthermore, as shown in Section~\ref{subsec:2d_rotary}, the 2D rotary positional encoding can be seamlessly integrated into existing language models, enabling a unified treatment of both sequential and temporal dimensions..

Given these equivalences, one can take a pretrained LLM (trained in a standard causal manner) and further fine-tune it with the CaDDi-AR diffusion objective\footnote{Pretrained causal LLM can also be fine-tuned as vanilla CaDDi models, using a simple causal bidirectional augmentation trick. See Appendix~\ref{sec:implement}.}. By allowing the model to condition on historical latent variables \(\mathbf{x}_{t+1:T}\), we endow it with iterative denoising and bidirectional modeling capabilities. This straightforward adaptation expands the model’s generation modes (e.g., text infilling) while fully leveraging the pretrained language model’s existing knowledge.

While several recent works have fine-tuned autoregressive LLMs as discrete diffusion models~\citep{dream2025, tae2025tess2largescalegeneralist, cetin2025largelanguagemodelsdiffusion, gong2024scaling}, our approach preserves the natural causal masking across both temporal and token dimensions. This design choice significantly mitigates the model drift issues observed in prior methods that replaced causal masks with full attention, which often required careful learning rate tuning to stabilize training~\citep{dream2025}.






\section{Experiments}

\paragraph{Baseline.} For most of our benchmark we compare \method{} against several well-established discrete diffusion models: D3PM \citep{d3pm}, SEDD \citep{lou2024discretediffusionmodelingestimating}, MDLM \citep{mdlm}, UDLM \citep{udlm}, and Discrete Flow Matching \citep{gat2024discreteflowmatching}. We adopt the official MDLM and UDLM codebases, which also include reference implementations of D3PM and SEDD. All models use a 12-layer Transformer (hidden size 768, 12 attention heads), trained with a learning rate of 3e-4, 2500 warm-up steps, and linear learning rate annealing. For evaluation, we sample the same number of sequences across all models. On the Text8 benchmark, we additionally include flow-based methods IAF/SCF \citep{iaf}, Argmax Flow \citep{argmax}, Discrete Flow \citep{discreteflow}, Any-order Autoregressive Models \citep{ardm}, MAC \citep{shih2022training}, and Mult. Diffusion \citep{argmax}. We denote our block-level autoregressive variant as CaDDi and the token-level variant as CaDDi-AR. Unless otherwise noted, all models are trained from scratch and experiments by default use absorbing-marginal kernels, corresponding to the linear noise schedule of D3PM \citep{d3pm}.

\begin{table*}[h]
    \centering
    \caption{{Evaluation on the LM1B dataset.} We report guided Generative Perplexity (PPL) under three pretrained causal language models (GPT-2, Llama-2-7B, and Llama-3.2-3B) at different sampling temperatures of each baseline model and \method{} ($T=1$, $T=0.7$, $T=0.5$). The best performance is \textbf{bolded}, and the second-best is \underline{underlined}.}
    \vspace{3pt}
    \begin{adjustbox}{max width=\linewidth}
    \begin{tabular}{lccc|ccc|ccc|c}
    \toprule
         & \multicolumn{3}{c}{\textsc{GPT2}} 
         & \multicolumn{3}{c}{\textsc{Llama-2}} 
         & \multicolumn{3}{c}{\textsc{Llama-3}}  &\multirow{2}{*}{\textsc{Entropy}}\\
    \cmidrule(lr){2-4}\cmidrule(lr){5-7}\cmidrule(lr){8-10}
    \textbf{Model} & \textsc{T=1} & \textsc{T=0.7} & \textsc{T=0.5} & \textsc{T=1} & \textsc{T=0.7} & \textsc{T=0.5} & \textsc{T=1} & \textsc{T=0.7} & \textsc{T=0.5} \\
    \midrule
        {Data} 
            & 114.79 & --- & --- & 22.97 & --- & --- & 42.20 & --- & --- & --- \\
    \midrule
        UDLM 
            & 313.24 & 318.80 & 328.99 & 110.37 & 111.86 & 119.21 & 207.91 & 219.86 & 231.05 & 6.0 \\
        D3PM     
            & 232.37 & 134.55 & 133.38 & 76.00 & 55.54 & 59.02 & 145.96 & 98.26 & 110.86 & 5.7\\
        SEDD      
            & 201.19 & 148.43 & {81.44} & 77.54 & 66.54 & {46.91} & 144.23 & 84.20 & {60.00} & 5.6\\
        MDLM 
            & 199.45 & 135.85 & 106.20 & 67.86 & 62.34 & 60.09 & 126.35 & 113.19 & 104.71 & 5.6\\
        DFM
            & {182.21} & {94.46} & 106.03 & \underline{67.02} & {38.93} & 61.91 & \textbf{120.09} & {66.22} & 102.89 & 5.7\\
    \midrule
        \textbf{\method{}}  
            &\underline{142.51} & \textbf{64.27} &  \textbf{45.96}
            & 70.27 & \underline{35.40} &  \textbf{23.81}
            & 124.51 & \textbf{58.26} & \textbf{36.79}& 5.7\\
        \textbf{\method{}-AR}  
            & \textbf{139.80} & \underline{76.02} & \underline{67.59} 
            & \textbf{65.93} & \textbf{35.38} & \underline{27.76} 
            & \underline{121.25} & \underline{59.66} & \underline{44.54} & 5.7\\

    \bottomrule
    \end{tabular}
    \end{adjustbox}
    \label{tab:lm1b}
\end{table*}

\paragraph{One Billion Words Dataset}
We evaluate \method{}'s generative capabilities on the One Billion Words dataset (LM1B) \citep{chelba2014billionwordbenchmarkmeasuring}, a large-scale natural language corpus comprising over 30 million English sentences of varying lengths. We follow the tokenization and training setup introduced in DiffusionBERT \citep{he2022diffusionbertimprovinggenerativemasked}. For UDLM, we use the pretrained weights provided by the authors. For other baselines, we retrain the models using their official codebases, as checkpoints are not publicly available.

To evaluate the quality of model-generated text, we report generative perplexity (Gen PPL), computed using several large language models as oracles, including GPT-2 \citep{radford2019language}, Llama-2 (7B) \citep{touvron2023llama2openfoundation}, and Llama-3 (3B) \citep{dubey2024llama}. All oracle models are pretrained on large-scale natural language corpora. To assess the diversity of generated outputs, we additionally compute the entropy of the generated text set without applying temperature scaling. Further details on these metrics are provided in Appendix~\ref{sec:appendix-metrics}.

As shown in Table~\ref{tab:lm1b}, CaDDi-AR and CaDDi consistently outperform baselines in generative perplexity across the three language model oracles. Notably, at low temperature, CaDDi alone performs comparably to—or even better than—CaDDi-AR. We hypothesize that this is due to the long-tail issue being more pronounced in block generation with fewer steps, where lower temperature helps mitigate sampling noise. Importantly, our approach maintains diversity, as indicated by comparable entropy scores.

\vspace{-5pt}
\paragraph{Text8 Dataset}
Following prior work~\citep{d3pm, shi2024simplified}, we trained a discrete diffusion model on short text chunks of length 256 from the text8 dataset. We use the same dataset split as in previous studies, training on the training set and reporting performance on the test set using the standard bits-per-dimension (BPD) metric. The BPD is defined as: $-\frac{1}{L} \sum_{i=1}^L \log _2 p\left(\mathbf{x}_i\right)$. Note that the bound of log-likelihood depends on the discretization of the diffusion schedule. To enable fair comparisons, we report the number of discretization steps used for each diffusion-based model.

We omit BPD results for CaDDi-AR due to the additional approximation error of likelihood introduced by its token-level autoregressive decomposition. As shown in Table~\ref{tab:text8}, under the same discretization setting (64 steps), CaDDi outperforms other discrete diffusion models. While any-order autoregressive models can be viewed as a special case of discrete diffusion with 256 steps, our 64-step model already approaches their performance.
\vspace{-5pt}
\begin{table}[h]
\caption{
\textbf{Evaluation results on Text8 Dataset} reporting layers, steps/discretization, bits-per-dimension (BPD), perplexity, and negative log-likelihood (NLL) for all compared models. For diffusion-based and any-order models, the values represent variational upper bounds.
Rows shaded in gray correspond to evaluations at 64 steps, enabling a direct comparison with CaDDi.
}
\centering
\resizebox{\textwidth}{!}{%
\begin{tabular}{lc ccccc}
\toprule
\textbf{Model} & \textbf{Layers} & \textbf{Steps/Discretization} & \textbf{BPD}\(\downarrow\) & \textbf{Perplexity}\(\downarrow\) & \textbf{NLL}\(\downarrow\) \\
\midrule
\textit{Autoregressive} \\
\quad IAF/SCF & 12 & - & 1.88 & 3.68 & 1.30 \\
\quad Argmax Flow & 12 & - & 1.39 & 2.62 & 0.96 \\
\quad Discrete Flow & 3$\times$8 & - & \textbf{1.23} & \textbf{2.35} & \textbf{0.85}\\
\quad Autoregressive & 12 & - & \textbf{1.23} & \textbf{2.35} & \textbf{0.85} \\
\midrule
\textit{Any Order Autoregressive} \\
\quad ARDM & 12 & - & $\leq$1.43 & $\leq$2.69 & $\leq$0.99 \\
\quad MAC & 12 & - & $\leq$1.40 & $\leq$2.64 & $\leq$0.97 \\
\midrule
\textit{Diffusion} \\
\quad Mult. Diffusion & 12 & 1000 & $\leq$1.72 & $\leq$3.29 & $\leq$1.19 \\
\quad D3PM Absorb & 12 & $\infty$ (continuous) & $\leq$1.45 & $\leq$2.73 & $\leq$1.01 \\
\rowcolor{gray!20} \quad & 12 & 64 & $\leq$1.51 & $\leq$2.85 & $\leq$1.05 \\
\quad SEDD Absorb & 12 & $\infty$ (continuous) & $\leq$\underline{1.39} & $\leq$\underline{2.62} & $\leq$\underline{0.96} \\
\rowcolor{gray!20} \quad & 12 & 64 & $\leq$1.46 & $\leq$2.75 & $\leq$1.01 \\
\quad UDLM & 12 & $\infty$ (continuous) & $\leq$1.44 & $\leq$2.71 & $\leq$1.00 \\
\rowcolor{gray!20} \quad & 12 & 64 & $\leq$1.60 & $\leq$3.03 & $\leq$1.11 \\
\quad MDLM & 12 & $\infty$ (continuous) & $\leq$1.40 & $\leq$2.64 & $\leq$0.97 \\
\rowcolor{gray!20} \quad & 12 & 64 & $\leq$1.46 & $\leq$2.75 & $\leq$1.01 \\
\rowcolor{gray!40} \quad CaDDi & 12 & 64 & $\leq$\textbf{1.41} & $\leq$\textbf{2.66} & $\leq$\textbf{0.98} \\
\bottomrule
\end{tabular}%
}
\label{tab:text8}
\vspace{-10pt}
\end{table}

\paragraph{General Reasoning Datasets with Fine-tuned LLM}

As shown in Section~\ref{sec:finetune}, CaDDi-AR retains compatibility with standard causal language models and can be fine-tuned from pretrained LLM checkpoints without architectural modifications. To evaluate its general reasoning ability, we fine-tune a 1.5B QWen model using the CaDDi-AR diffusion objective on a range of natural language understanding benchmarks, including ARC-Challenge, ARC-Easy~\citep{arc}, BoolQ~\citep{clark2019boolq}, PIQA~\citep{Bisk2020}, RACE~\citep{lai-etal-2017-race}, Social IQA~\citep{sap2019socialiqacommonsensereasoningsocial}, and LAMBADA~\citep{lambada_dataset}. These tasks span multiple reasoning paradigms, from commonsense inference to multi-hop reading comprehension. We benchmark with model of comparable size, including causal language model QWen2~\citep{qwen2}, GPT2~\citep{radford2019language}, TinyLlama~\citep{zhang2024tinyllamaopensourcesmalllanguage} and a discrete diffusion model MDM~\citep{nie2025scalingmaskeddiffusionmodels} trained from scratch. 

As shown in Table~\ref{tab:main-results}, CaDDi-AR achieves consistent gains over  diffusion-based baselines of comparable size. It also include +1.9\% improvement on ARC-Challenge and a +2.4\% gain on LAMBADA over the base QWen model. These results suggest that CaDDi-AR’s non-Markovian formulation supports more robust and revisable reasoning, while maintaining compatibility with existing language model infrastructures. It 's worth noting that compared with MDM, CaDDi-AR requires much less training effort.  
\vspace{-10pt}
\begin{table}[t]
\centering
\caption{{Accuracy (\%) on a range of natural language reasoning benchmarks}.}
\label{tab:main-results}
\resizebox{0.99\textwidth}{!}{  
\begin{tabular}{lcccccccc}
\toprule
\textbf{Model} & \textbf{Size} & \textbf{ARC-Chal.} & \textbf{ARC-Easy} & \textbf{BoolQ} & \textbf{PIQA} & \textbf{RACE} & \textbf{Social IQA} & \textbf{LAMBADA} \\
\midrule
QWen2-1.5B & 1.5B & 33.7 & 66.1 & \textbf{72.6} & 75.4 & \textbf{36.6} & 45.8 & 63.9 \\
\midrule
GPT-2  & 1.5B & 28.5  & 51.1 & 61.8 & 70.5 & 33.1 & 40.3 & 44.6 \\
TinyLlama & 1.1B & 29.9 & 52.2 & 59.4 & 70.3 & 35.6 & 39.4 & 43.2 \\
MDM & 1.1B & - & 48.7 & 62.2 & 69.5 & 35.6 & 41.0 & 52.7 \\
\midrule
CaDDi-AR & 1.5B & \textbf{34.2} & \textbf{67.8} & 71.6 & 75.4 & 34.3 & \textbf{46.9} & \textbf{66.3} \\
\bottomrule
\vspace{-10pt}
\end{tabular}
\vspace{-10pt}
}
\end{table}


\vspace{-1pt}
\paragraph{Conditional Text Generation}
\begin{wraptable}{r}{0.55\textwidth}
\vspace{-10pt}  
\caption{{Comparison of GPT-2, CaDDi, and CaDDi-CFG on the Amazon Polarity dataset.} Generation is
  conditioned on either positive or negative.}
\vspace{3pt}
\centering
\begin{adjustbox}{max width=0.55\textwidth}
\begin{tabular}{lcc}
\toprule
\textbf{Model} & {Condition} & \textbf{Sentiment Accuracy (\%)} \\ 
\midrule
\textbf{GPT-2}  & Positive & 73.07 \\
                & Negative & 75.18 \\
\midrule
\textbf{CaDDi-CFG\(_{\gamma=1}\)}  & Positive & 71.37 \\
                & Negative & 85.42 \\
\midrule
\textbf{CaDDi-CFG\(_{\gamma=1.25}\)} & Positive & 73.61 \\
                & Negative & 85.92 \\
\bottomrule
\end{tabular}
\end{adjustbox}
\label{tab:sentiment-comparison}
\vspace{-10pt}  
\end{wraptable}

We also evaluate conditional text generation on the Amazon Polarity dataset~\citep{mcauley2013hidden}, which consists of 3.6M Amazon reviews labeled as positive or negative. We adapt this task as text infilling by prepending a label-based prompt to each review (see Appendix~\ref{sec:amazon_process} for details). We train the conditional generator \(p_\theta\left(\mathbf{x}_0 \mid \mathbf{x}_t, y\right)\) alongside unconditional one  \(p_\theta\left(\mathbf{x}_0 \mid \mathbf{x}_t\right)\), by preserving certain parts of the text as fixed.

We measure sentiment accuracy (SA) using a fine-tuned DistilBERT classifier. As shown in Table~\ref{tab:sentiment-comparison}, our approach achieves performance comparable to a fine-tuned GPT-2 on the same dataset while offering more flexible generation (unlike GPT2 only allow prompting from begining, our method allows for prompting from arbitrary parts of the text, such as the middle or the title. Examples are shown in Figure~\ref{fig:amazon_result_negative}). Furthermore, by applying classifier-free guidance (denoted as CaDDi-CFG)~\citep{ho2022classifier, udlm} with different guidance scales \(\gamma\), we generate reviews that better align with the given prompts.
\vspace{-10pt}

\paragraph{Additional Analysis and Ablation Study.}
\label{sec:robust}
\begin{wrapfigure}{r}{0.4\linewidth}
    \vspace{-20pt}
    \centering
    \includegraphics[width=\linewidth]{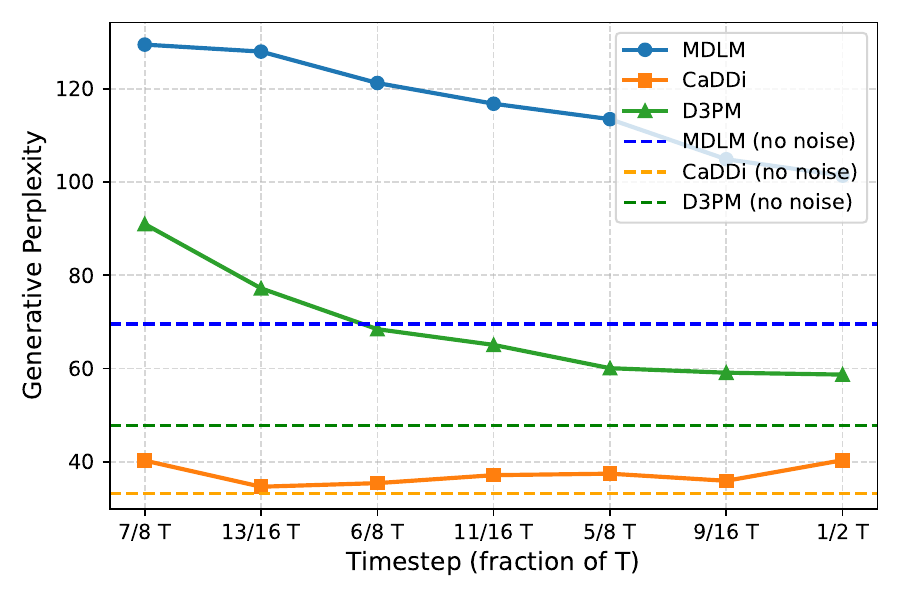}
    \caption{{Generation performance under manually injected noise at different timestep}}
    \label{fig:robust}
    \vspace{-10pt}
\end{wrapfigure}
To assess the inference robustness of discrete diffusion models, we perform an additional ablation study by manually injecting controlled noise (i.e., incorrect predictions) at various timesteps during generation, simulating potential inference-time errors. As shown in Figure~\ref{fig:robust}, introducing noise at earlier stages generally leads to greater error accumulation. Nonetheless, our proposed CaDDi framework consistently exhibits stronger resilience compared to D3PM and MDLM, maintaining higher generation quality under perturbations across all settings. We remain the full ablation study and discussion in Appendix~\ref{sec:addition_experiment}.

\vspace{-5pt}
\section{Conclusion}
\vspace{-5pt}
We introduced \emph{CaDDi}, a causal discrete diffusion framework that relaxes the traditional Markovian assumptions in favor of an autoregressive inference process. By explicitly conditioning each denoising step on the entire future trajectory, CaDDi captures richer temporal dependencies and leverages iterative refinement. Critically, our approach can also be built atop existing causal language models---bridging standard sequence modeling with powerful diffusion capabilities---while preserving both knowledge from large-scale training and the flexibility of iterative editing.


\section*{Acknowledgement}
The authors thank collaborators and contributors from across institutions for their invaluable support and insights throughout this project. This work was supported in part by the National Institutes of Health (NIH) grant R35GM143072–01, the National Science Foundation (NSF) grant IIS Division of Information \& Intelligent Systems 2443528, and the Yale Colton Center Award, all awarded to Prof. David van Dijk; and by the National Science Foundation (NSF) grants IIS Division of Information \& Intelligent Systems 2403317 and CNS Division of Computer and Network Systems 2431504, as well as the Envisioning Artificial Intelligence at Yale 2025 initiative from the Yale Office of the Provost, awarded to Prof. Rex Ying.
\bibliographystyle{plainnat} 
\bibliography{references}

\begin{thebibliography}{67}
\providecommand{\natexlab}[1]{#1}
\providecommand{\url}[1]{\texttt{#1}}
\expandafter\ifx\csname urlstyle\endcsname\relax
  \providecommand{\doi}[1]{doi: #1}\else
  \providecommand{\doi}{doi: \begingroup \urlstyle{rm}\Url}\fi

\bibitem[qwe(2024)]{qwen2}
Qwen2 technical report.
\newblock 2024.

\bibitem[Austin et~al.(2023)Austin, Johnson, Ho, Tarlow, and van~den Berg]{d3pm}
Jacob Austin, Daniel~D. Johnson, Jonathan Ho, Daniel Tarlow, and Rianne van~den Berg.
\newblock Structured denoising diffusion models in discrete state-spaces, 2023.
\newblock URL \url{https://arxiv.org/abs/2107.03006}.

\bibitem[Biderman et~al.(2023)Biderman, Schoelkopf, Anthony, Bradley, O’Brien, Hallahan, Khan, Purohit, Prashanth, Raff, et~al.]{biderman2023pythia}
Stella Biderman, Hailey Schoelkopf, Quentin~Gregory Anthony, Herbie Bradley, Kyle O’Brien, Eric Hallahan, Mohammad~Aflah Khan, Shivanshu Purohit, USVSN~Sai Prashanth, Edward Raff, et~al.
\newblock Pythia: A suite for analyzing large language models across training and scaling.
\newblock In \emph{International Conference on Machine Learning}, pages 2397--2430. PMLR, 2023.

\bibitem[Bisk et~al.(2020)Bisk, Zellers, Bras, Gao, and Choi]{Bisk2020}
Yonatan Bisk, Rowan Zellers, Ronan~Le Bras, Jianfeng Gao, and Yejin Choi.
\newblock Piqa: Reasoning about physical commonsense in natural language.
\newblock In \emph{Thirty-Fourth AAAI Conference on Artificial Intelligence}, 2020.

\bibitem[Campbell et~al.(2022)Campbell, Benton, De~Bortoli, Rainforth, Deligiannidis, and Doucet]{campbell2022continuous}
Andrew Campbell, Joe Benton, Valentin De~Bortoli, Thomas Rainforth, George Deligiannidis, and Arnaud Doucet.
\newblock A continuous time framework for discrete denoising models.
\newblock \emph{Advances in Neural Information Processing Systems}, 35:\penalty0 28266--28279, 2022.

\bibitem[Cetin et~al.(2025)Cetin, Zhao, and Tang]{cetin2025largelanguagemodelsdiffusion}
Edoardo Cetin, Tianyu Zhao, and Yujin Tang.
\newblock Large language models to diffusion finetuning, 2025.
\newblock URL \url{https://arxiv.org/abs/2501.15781}.

\bibitem[Chelba et~al.(2014)Chelba, Mikolov, Schuster, Ge, Brants, Koehn, and Robinson]{chelba2014billionwordbenchmarkmeasuring}
Ciprian Chelba, Tomas Mikolov, Mike Schuster, Qi~Ge, Thorsten Brants, Phillipp Koehn, and Tony Robinson.
\newblock One billion word benchmark for measuring progress in statistical language modeling, 2014.
\newblock URL \url{https://arxiv.org/abs/1312.3005}.

\bibitem[Chen et~al.(2023)Chen, Borgeaud, Irving, Lespiau, Sifre, and Jumper]{chen2023accelerating}
Charlie Chen, Sebastian Borgeaud, Geoffrey Irving, Jean-Baptiste Lespiau, Laurent Sifre, and John Jumper.
\newblock Accelerating large language model decoding with speculative sampling.
\newblock \emph{arXiv preprint arXiv:2302.01318}, 2023.

\bibitem[Chen et~al.(2020)Chen, Radford, Child, Wu, Jun, Luan, and Sutskever]{chen2020generative}
Mark Chen, Alec Radford, Rewon Child, Jeffrey Wu, Heewoo Jun, David Luan, and Ilya Sutskever.
\newblock Generative pretraining from pixels.
\newblock In \emph{International conference on machine learning}, pages 1691--1703. PMLR, 2020.

\bibitem[Chen et~al.(2024)Chen, Yuan, Li, Kou, Zhang, and Gu]{chen2024fast}
Zixiang Chen, Huizhuo Yuan, Yongqian Li, Yiwen Kou, Junkai Zhang, and Quanquan Gu.
\newblock Fast sampling via discrete non-markov diffusion models with predetermined transition time.
\newblock In \emph{The Thirty-eighth Annual Conference on Neural Information Processing Systems}, 2024.

\bibitem[Child et~al.(2019)Child, Gray, Radford, and Sutskever]{child2019generatinglongsequencessparse}
Rewon Child, Scott Gray, Alec Radford, and Ilya Sutskever.
\newblock Generating long sequences with sparse transformers, 2019.
\newblock URL \url{https://arxiv.org/abs/1904.10509}.

\bibitem[Chowdhery et~al.(2023)Chowdhery, Narang, Devlin, Bosma, Mishra, Roberts, Barham, Chung, Sutton, Gehrmann, et~al.]{chowdhery2023palm}
Aakanksha Chowdhery, Sharan Narang, Jacob Devlin, Maarten Bosma, Gaurav Mishra, Adam Roberts, Paul Barham, Hyung~Won Chung, Charles Sutton, Sebastian Gehrmann, et~al.
\newblock Palm: Scaling language modeling with pathways.
\newblock \emph{Journal of Machine Learning Research}, 24\penalty0 (240):\penalty0 1--113, 2023.

\bibitem[Clark et~al.(2019)Clark, Lee, Chang, Kwiatkowski, Collins, and Toutanova]{clark2019boolq}
Christopher Clark, Kenton Lee, Ming-Wei Chang, Tom Kwiatkowski, Michael Collins, and Kristina Toutanova.
\newblock Boolq: Exploring the surprising difficulty of natural yes/no questions.
\newblock \emph{arXiv preprint arXiv:1905.10044}, 2019.

\bibitem[Clark et~al.(2018)Clark, Cowhey, Etzioni, Khot, Sabharwal, Schoenick, and Tafjord]{arc}
Peter Clark, Isaac Cowhey, Oren Etzioni, Tushar Khot, Ashish Sabharwal, Carissa Schoenick, and Oyvind Tafjord.
\newblock Think you have solved question answering? try arc, the ai2 reasoning challenge.
\newblock \emph{arXiv preprint arXiv:1803.05457}, 2018.

\bibitem[Dao et~al.(2022)Dao, Fu, Ermon, Rudra, and R{\'e}]{dao2022flashattention}
Tri Dao, Daniel~Y. Fu, Stefano Ermon, Atri Rudra, and Christopher R{\'e}.
\newblock Flash{A}ttention: Fast and memory-efficient exact attention with {IO}-awareness.
\newblock In \emph{Advances in Neural Information Processing Systems (NeurIPS)}, 2022.

\bibitem[Davis et~al.(2024)Davis, Kessler, Petrache, İsmail~İlkan Ceylan, Bronstein, and Bose]{davis2024fisherflowmatchinggenerative}
Oscar Davis, Samuel Kessler, Mircea Petrache, İsmail~İlkan Ceylan, Michael Bronstein, and Avishek~Joey Bose.
\newblock Fisher flow matching for generative modeling over discrete data, 2024.
\newblock URL \url{https://arxiv.org/abs/2405.14664}.

\bibitem[Dieleman et~al.(2022)Dieleman, Sartran, Roshannai, Savinov, Ganin, Richemond, Doucet, Strudel, Dyer, Durkan, et~al.]{dieleman2022continuous}
Sander Dieleman, Laurent Sartran, Arman Roshannai, Nikolay Savinov, Yaroslav Ganin, Pierre~H Richemond, Arnaud Doucet, Robin Strudel, Chris Dyer, Conor Durkan, et~al.
\newblock Continuous diffusion for categorical data.
\newblock \emph{arXiv preprint arXiv:2211.15089}, 2022.

\bibitem[Dubey et~al.(2024)Dubey, Jauhri, Pandey, Kadian, Al-Dahle, Letman, Mathur, Schelten, Yang, Fan, et~al.]{dubey2024llama}
Abhimanyu Dubey, Abhinav Jauhri, Abhinav Pandey, Abhishek Kadian, Ahmad Al-Dahle, Aiesha Letman, Akhil Mathur, Alan Schelten, Amy Yang, Angela Fan, et~al.
\newblock The llama 3 herd of models.
\newblock \emph{arXiv preprint arXiv:2407.21783}, 2024.

\bibitem[Gao et~al.(2024)Gao, Tow, Abbasi, Biderman, Black, DiPofi, Foster, Golding, Hsu, Le~Noac'h, Li, McDonell, Muennighoff, Ociepa, Phang, Reynolds, Schoelkopf, Skowron, Sutawika, Tang, Thite, Wang, Wang, and Zou]{eval-harness}
Leo Gao, Jonathan Tow, Baber Abbasi, Stella Biderman, Sid Black, Anthony DiPofi, Charles Foster, Laurence Golding, Jeffrey Hsu, Alain Le~Noac'h, Haonan Li, Kyle McDonell, Niklas Muennighoff, Chris Ociepa, Jason Phang, Laria Reynolds, Hailey Schoelkopf, Aviya Skowron, Lintang Sutawika, Eric Tang, Anish Thite, Ben Wang, Kevin Wang, and Andy Zou.
\newblock The language model evaluation harness, 07 2024.
\newblock URL \url{https://zenodo.org/records/12608602}.

\bibitem[Gardiner()]{gardinerstochastic}
C~Gardiner.
\newblock Stochastic methods: A handbook for the natural and social sciences2009.

\bibitem[Gat et~al.(2024)Gat, Remez, Shaul, Kreuk, Chen, Synnaeve, Adi, and Lipman]{gat2024discreteflowmatching}
Itai Gat, Tal Remez, Neta Shaul, Felix Kreuk, Ricky T.~Q. Chen, Gabriel Synnaeve, Yossi Adi, and Yaron Lipman.
\newblock Discrete flow matching, 2024.
\newblock URL \url{https://arxiv.org/abs/2407.15595}.

\bibitem[Gong et~al.(2024)Gong, Agarwal, Zhang, Ye, Zheng, Li, An, Zhao, Bi, Han, et~al.]{gong2024scaling}
Shansan Gong, Shivam Agarwal, Yizhe Zhang, Jiacheng Ye, Lin Zheng, Mukai Li, Chenxin An, Peilin Zhao, Wei Bi, Jiawei Han, et~al.
\newblock Scaling diffusion language models via adaptation from autoregressive models.
\newblock \emph{arXiv preprint arXiv:2410.17891}, 2024.

\bibitem[Gu et~al.(2024)Gu, Wang, Zhang, Zhang, Zhang, Jaitly, Susskind, and Zhai]{DART}
Jiatao Gu, Yuyang Wang, Yizhe Zhang, Qihang Zhang, Dinghuai Zhang, Navdeep Jaitly, Josh Susskind, and Shuangfei Zhai.
\newblock Dart: Denoising autoregressive transformer for scalable text-to-image generation, 2024.
\newblock URL \url{https://arxiv.org/abs/2410.08159}.

\bibitem[Gulrajani and Hashimoto(2024)]{gulrajani2024likelihood}
Ishaan Gulrajani and Tatsunori~B Hashimoto.
\newblock Likelihood-based diffusion language models.
\newblock \emph{Advances in Neural Information Processing Systems}, 36, 2024.

\bibitem[He et~al.(2024)He, Levine, Vrkic, Bressana, Zhang, Rizvi, Zhang, Zappala, and van Dijk]{he2024calmflowvolterraflowmatching}
Sizhuang He, Daniel Levine, Ivan Vrkic, Marco~Francesco Bressana, David Zhang, Syed~Asad Rizvi, Yangtian Zhang, Emanuele Zappala, and David van Dijk.
\newblock Calmflow: Volterra flow matching using causal language models, 2024.
\newblock URL \url{https://arxiv.org/abs/2410.05292}.

\bibitem[He et~al.(2022)He, Sun, Wang, Huang, and Qiu]{he2022diffusionbertimprovinggenerativemasked}
Zhengfu He, Tianxiang Sun, Kuanning Wang, Xuanjing Huang, and Xipeng Qiu.
\newblock Diffusionbert: Improving generative masked language models with diffusion models, 2022.
\newblock URL \url{https://arxiv.org/abs/2211.15029}.

\bibitem[Ho and Salimans(2022)]{ho2022classifier}
Jonathan Ho and Tim Salimans.
\newblock Classifier-free diffusion guidance.
\newblock \emph{arXiv preprint arXiv:2207.12598}, 2022.

\bibitem[Ho et~al.(2020)Ho, Jain, and Abbeel]{ho2020denoisingdiffusionprobabilisticmodels}
Jonathan Ho, Ajay Jain, and Pieter Abbeel.
\newblock Denoising diffusion probabilistic models, 2020.
\newblock URL \url{https://arxiv.org/abs/2006.11239}.

\bibitem[Holtzman et~al.(2020)Holtzman, Buys, Du, Forbes, and Choi]{holtzman2020curiouscaseneuraltext}
Ari Holtzman, Jan Buys, Li~Du, Maxwell Forbes, and Yejin Choi.
\newblock The curious case of neural text degeneration, 2020.
\newblock URL \url{https://arxiv.org/abs/1904.09751}.

\bibitem[Hoogeboom et~al.(2021{\natexlab{a}})Hoogeboom, Gritsenko, Bastings, Poole, Berg, and Salimans]{ardm}
Emiel Hoogeboom, Alexey~A Gritsenko, Jasmijn Bastings, Ben Poole, Rianne van~den Berg, and Tim Salimans.
\newblock Autoregressive diffusion models.
\newblock \emph{arXiv preprint arXiv:2110.02037}, 2021{\natexlab{a}}.

\bibitem[Hoogeboom et~al.(2021{\natexlab{b}})Hoogeboom, Nielsen, Jaini, Forr{\'e}, and Welling]{argmax}
Emiel Hoogeboom, Didrik Nielsen, Priyank Jaini, Patrick Forr{\'e}, and Max Welling.
\newblock Argmax flows and multinomial diffusion: Learning categorical distributions.
\newblock \emph{Advances in neural information processing systems}, 34:\penalty0 12454--12465, 2021{\natexlab{b}}.

\bibitem[Hu and Ommer(2024)]{hu2024mask}
Vincent~Tao Hu and Bj{\"o}rn Ommer.
\newblock [mask] is all you need.
\newblock \emph{arXiv preprint arXiv:2412.06787}, 2024.

\bibitem[Lai et~al.(2017)Lai, Xie, Liu, Yang, and Hovy]{lai-etal-2017-race}
Guokun Lai, Qizhe Xie, Hanxiao Liu, Yiming Yang, and Eduard Hovy.
\newblock {RACE}: Large-scale {R}e{A}ding comprehension dataset from examinations.
\newblock In \emph{Proceedings of the 2017 Conference on Empirical Methods in Natural Language Processing}, pages 785--794, Copenhagen, Denmark, September 2017. Association for Computational Linguistics.
\newblock \doi{10.18653/v1/D17-1082}.
\newblock URL \url{https://aclanthology.org/D17-1082}.

\bibitem[Leviathan et~al.(2023)Leviathan, Kalman, and Matias]{leviathan2023fast}
Yaniv Leviathan, Matan Kalman, and Yossi Matias.
\newblock Fast inference from transformers via speculative decoding.
\newblock In \emph{International Conference on Machine Learning}, pages 19274--19286. PMLR, 2023.

\bibitem[Lipman et~al.(2023)Lipman, Chen, Ben-Hamu, Nickel, and Le]{lipman2023flowmatchinggenerativemodeling}
Yaron Lipman, Ricky T.~Q. Chen, Heli Ben-Hamu, Maximilian Nickel, and Matt Le.
\newblock Flow matching for generative modeling, 2023.
\newblock URL \url{https://arxiv.org/abs/2210.02747}.

\bibitem[Liu et~al.(2024)Liu, Broadrick, Niepert, and Broeck]{liu2024discrete}
Anji Liu, Oliver Broadrick, Mathias Niepert, and Guy Van~den Broeck.
\newblock Discrete copula diffusion.
\newblock \emph{arXiv preprint arXiv:2410.01949}, 2024.

\bibitem[Lou et~al.(2024)Lou, Meng, and Ermon]{lou2024discretediffusionmodelingestimating}
Aaron Lou, Chenlin Meng, and Stefano Ermon.
\newblock Discrete diffusion modeling by estimating the ratios of the data distribution, 2024.
\newblock URL \url{https://arxiv.org/abs/2310.16834}.

\bibitem[McAuley and Leskovec(2013)]{mcauley2013hidden}
Julian McAuley and Jure Leskovec.
\newblock Hidden factors and hidden topics: understanding rating dimensions with review text.
\newblock In \emph{Proceedings of the 7th ACM conference on Recommender systems}, pages 165--172, 2013.

\bibitem[Nie et~al.(2025)Nie, Zhu, Du, Pang, Liu, Zeng, Lin, and Li]{nie2025scalingmaskeddiffusionmodels}
Shen Nie, Fengqi Zhu, Chao Du, Tianyu Pang, Qian Liu, Guangtao Zeng, Min Lin, and Chongxuan Li.
\newblock Scaling up masked diffusion models on text, 2025.
\newblock URL \url{https://arxiv.org/abs/2410.18514}.

\bibitem[Paperno et~al.(2016)Paperno, Kruszewski, Lazaridou, Pham, Bernardi, Pezzelle, Baroni, Boleda, and Fernandez]{lambada_dataset}
Denis Paperno, Germ\'{a}n Kruszewski, Angeliki Lazaridou, Ngoc~Quan Pham, Raffaella Bernardi, Sandro Pezzelle, Marco Baroni, Gemma Boleda, and Raquel Fernandez.
\newblock The {LAMBADA} dataset: Word prediction requiring a broad discourse context.
\newblock In \emph{Proceedings of the 54th Annual Meeting of the Association for Computational Linguistics (Volume 1: Long Papers)}, pages 1525--1534, Berlin, Germany, August 2016. Association for Computational Linguistics.
\newblock URL \url{http://www.aclweb.org/anthology/P16-1144}.

\bibitem[Park et~al.(2024)Park, Lai, Hayakawa, Takida, and Mitsufuji]{park2024textit}
Yong-Hyun Park, Chieh-Hsin Lai, Satoshi Hayakawa, Yuhta Takida, and Yuki Mitsufuji.
\newblock Jump your steps: Optimizing sampling schedule of discrete diffusion models.
\newblock \emph{arXiv preprint arXiv:2410.07761}, 2024.

\bibitem[Radford et~al.(2019)Radford, Wu, Child, Luan, Amodei, and Sutskever]{radford2019language}
Alec Radford, Jeff Wu, Rewon Child, David Luan, Dario Amodei, and Ilya Sutskever.
\newblock Language models are unsupervised multitask learners.
\newblock 2019.

\bibitem[Ramesh et~al.(2021)Ramesh, Pavlov, Goh, Gray, Voss, Radford, Chen, and Sutskever]{ramesh2021zero}
Aditya Ramesh, Mikhail Pavlov, Gabriel Goh, Scott Gray, Chelsea Voss, Alec Radford, Mark Chen, and Ilya Sutskever.
\newblock Zero-shot text-to-image generation.
\newblock In \emph{International conference on machine learning}, pages 8821--8831. Pmlr, 2021.

\bibitem[Sahoo et~al.(2024)Sahoo, Arriola, Schiff, Gokaslan, Marroquin, Chiu, Rush, and Kuleshov]{mdlm}
Subham~Sekhar Sahoo, Marianne Arriola, Yair Schiff, Aaron Gokaslan, Edgar Marroquin, Justin~T Chiu, Alexander Rush, and Volodymyr Kuleshov.
\newblock Simple and effective masked diffusion language models.
\newblock \emph{arXiv preprint arXiv:2406.07524}, 2024.

\bibitem[Sap et~al.(2019)Sap, Rashkin, Chen, LeBras, and Choi]{sap2019socialiqacommonsensereasoningsocial}
Maarten Sap, Hannah Rashkin, Derek Chen, Ronan LeBras, and Yejin Choi.
\newblock Socialiqa: Commonsense reasoning about social interactions, 2019.
\newblock URL \url{https://arxiv.org/abs/1904.09728}.

\bibitem[Schiff et~al.(2024)Schiff, Sahoo, Phung, Wang, Boshar, Dalla-torre, de~Almeida, Rush, Pierrot, and Kuleshov]{udlm}
Yair Schiff, Subham~Sekhar Sahoo, Hao Phung, Guanghan Wang, Sam Boshar, Hugo Dalla-torre, Bernardo~P de~Almeida, Alexander Rush, Thomas Pierrot, and Volodymyr Kuleshov.
\newblock Simple guidance mechanisms for discrete diffusion models.
\newblock \emph{arXiv preprint arXiv:2412.10193}, 2024.

\bibitem[Shen et~al.(2023)Shen, Peng, Shen, Fu, Harchaoui, and Choi]{shen2023film}
Tianxiao Shen, Hao Peng, Ruoqi Shen, Yao Fu, Zaid Harchaoui, and Yejin Choi.
\newblock Film: Fill-in language models for any-order generation.
\newblock \emph{arXiv preprint arXiv:2310.09930}, 2023.

\bibitem[Shi et~al.(2024)Shi, Han, Wang, Doucet, and Titsias]{shi2024simplified}
Jiaxin Shi, Kehang Han, Zhe Wang, Arnaud Doucet, and Michalis~K Titsias.
\newblock Simplified and generalized masked diffusion for discrete data.
\newblock \emph{arXiv preprint arXiv:2406.04329}, 2024.

\bibitem[Shih et~al.(2022)Shih, Sadigh, and Ermon]{shih2022training}
Andy Shih, Dorsa Sadigh, and Stefano Ermon.
\newblock Training and inference on any-order autoregressive models the right way.
\newblock \emph{Advances in Neural Information Processing Systems}, 35:\penalty0 2762--2775, 2022.

\bibitem[Song et~al.(2021)Song, Sohl-Dickstein, Kingma, Kumar, Ermon, and Poole]{song2021scorebasedgenerativemodelingstochastic}
Yang Song, Jascha Sohl-Dickstein, Diederik~P. Kingma, Abhishek Kumar, Stefano Ermon, and Ben Poole.
\newblock Score-based generative modeling through stochastic differential equations, 2021.
\newblock URL \url{https://arxiv.org/abs/2011.13456}.

\bibitem[Springer et~al.(2024)Springer, Kotha, Fried, Neubig, and Raghunathan]{springer2024repetitionimproveslanguagemodel}
Jacob~Mitchell Springer, Suhas Kotha, Daniel Fried, Graham Neubig, and Aditi Raghunathan.
\newblock Repetition improves language model embeddings, 2024.
\newblock URL \url{https://arxiv.org/abs/2402.15449}.

\bibitem[Stark et~al.(2024)Stark, Jing, Wang, Corso, Berger, Barzilay, and Jaakkola]{stark2024dirichletflowmatchingapplications}
Hannes Stark, Bowen Jing, Chenyu Wang, Gabriele Corso, Bonnie Berger, Regina Barzilay, and Tommi Jaakkola.
\newblock Dirichlet flow matching with applications to dna sequence design, 2024.
\newblock URL \url{https://arxiv.org/abs/2402.05841}.

\bibitem[Su et~al.(2024)Su, Ahmed, Lu, Pan, Bo, and Liu]{roformer}
Jianlin Su, Murtadha Ahmed, Yu~Lu, Shengfeng Pan, Wen Bo, and Yunfeng Liu.
\newblock Roformer: Enhanced transformer with rotary position embedding.
\newblock \emph{Neurocomputing}, 568:\penalty0 127063, 2024.

\bibitem[Tae et~al.(2025)Tae, Ivison, Kumar, and Cohan]{tae2025tess2largescalegeneralist}
Jaesung Tae, Hamish Ivison, Sachin Kumar, and Arman Cohan.
\newblock Tess 2: A large-scale generalist diffusion language model, 2025.
\newblock URL \url{https://arxiv.org/abs/2502.13917}.

\bibitem[Tong et~al.(2024)Tong, Fatras, Malkin, Huguet, Zhang, Rector-Brooks, Wolf, and Bengio]{tong2024improvinggeneralizingflowbasedgenerative}
Alexander Tong, Kilian Fatras, Nikolay Malkin, Guillaume Huguet, Yanlei Zhang, Jarrid Rector-Brooks, Guy Wolf, and Yoshua Bengio.
\newblock Improving and generalizing flow-based generative models with minibatch optimal transport, 2024.
\newblock URL \url{https://arxiv.org/abs/2302.00482}.

\bibitem[Touvron et~al.(2023{\natexlab{a}})Touvron, Lavril, Izacard, Martinet, Lachaux, Lacroix, Rozi{\`e}re, Goyal, Hambro, Azhar, et~al.]{touvron2023llama}
Hugo Touvron, Thibaut Lavril, Gautier Izacard, Xavier Martinet, Marie-Anne Lachaux, Timoth{\'e}e Lacroix, Baptiste Rozi{\`e}re, Naman Goyal, Eric Hambro, Faisal Azhar, et~al.
\newblock Llama: Open and efficient foundation language models.
\newblock \emph{arXiv preprint arXiv:2302.13971}, 2023{\natexlab{a}}.

\bibitem[Touvron et~al.(2023{\natexlab{b}})Touvron, Martin, Stone, Albert, Almahairi, Babaei, Bashlykov, Batra, Bhargava, Bhosale, Bikel, Blecher, Ferrer, Chen, Cucurull, Esiobu, Fernandes, Fu, Fu, Fuller, Gao, Goswami, Goyal, Hartshorn, Hosseini, Hou, Inan, Kardas, Kerkez, Khabsa, Kloumann, Korenev, Koura, Lachaux, Lavril, Lee, Liskovich, Lu, Mao, Martinet, Mihaylov, Mishra, Molybog, Nie, Poulton, Reizenstein, Rungta, Saladi, Schelten, Silva, Smith, Subramanian, Tan, Tang, Taylor, Williams, Kuan, Xu, Yan, Zarov, Zhang, Fan, Kambadur, Narang, Rodriguez, Stojnic, Edunov, and Scialom]{touvron2023llama2openfoundation}
Hugo Touvron, Louis Martin, Kevin Stone, Peter Albert, Amjad Almahairi, Yasmine Babaei, Nikolay Bashlykov, Soumya Batra, Prajjwal Bhargava, Shruti Bhosale, Dan Bikel, Lukas Blecher, Cristian~Canton Ferrer, Moya Chen, Guillem Cucurull, David Esiobu, Jude Fernandes, Jeremy Fu, Wenyin Fu, Brian Fuller, Cynthia Gao, Vedanuj Goswami, Naman Goyal, Anthony Hartshorn, Saghar Hosseini, Rui Hou, Hakan Inan, Marcin Kardas, Viktor Kerkez, Madian Khabsa, Isabel Kloumann, Artem Korenev, Punit~Singh Koura, Marie-Anne Lachaux, Thibaut Lavril, Jenya Lee, Diana Liskovich, Yinghai Lu, Yuning Mao, Xavier Martinet, Todor Mihaylov, Pushkar Mishra, Igor Molybog, Yixin Nie, Andrew Poulton, Jeremy Reizenstein, Rashi Rungta, Kalyan Saladi, Alan Schelten, Ruan Silva, Eric~Michael Smith, Ranjan Subramanian, Xiaoqing~Ellen Tan, Binh Tang, Ross Taylor, Adina Williams, Jian~Xiang Kuan, Puxin Xu, Zheng Yan, Iliyan Zarov, Yuchen Zhang, Angela Fan, Melanie Kambadur, Sharan Narang, Aurelien Rodriguez, Robert Stojnic, Sergey Edunov, and Thomas
  Scialom.
\newblock Llama 2: Open foundation and fine-tuned chat models, 2023{\natexlab{b}}.
\newblock URL \url{https://arxiv.org/abs/2307.09288}.

\bibitem[Tran et~al.(2019)Tran, Vafa, Agrawal, Dinh, and Poole]{discreteflow}
Dustin Tran, Keyon Vafa, Kumar Agrawal, Laurent Dinh, and Ben Poole.
\newblock Discrete flows: Invertible generative models of discrete data.
\newblock \emph{Advances in Neural Information Processing Systems}, 32, 2019.

\bibitem[Van~Kampen(1992)]{van1992stochastic}
Nicolaas~Godfried Van~Kampen.
\newblock \emph{Stochastic processes in physics and chemistry}, volume~1.
\newblock Elsevier, 1992.

\bibitem[Vaswani(2017)]{vaswani2017attention}
A~Vaswani.
\newblock Attention is all you need.
\newblock \emph{Advances in Neural Information Processing Systems}, 2017.

\bibitem[Wang et~al.(2025)Wang, Schiff, Sahoo, and Kuleshov]{wang2025remasking}
Guanghan Wang, Yair Schiff, Subham~Sekhar Sahoo, and Volodymyr Kuleshov.
\newblock Remasking discrete diffusion models with inference-time scaling.
\newblock \emph{arXiv preprint arXiv:2503.00307}, 2025.

\bibitem[Xiao et~al.(2021)Xiao, Kreis, and Vahdat]{xiao2021tackling}
Zhisheng Xiao, Karsten Kreis, and Arash Vahdat.
\newblock Tackling the generative learning trilemma with denoising diffusion gans.
\newblock \emph{arXiv preprint arXiv:2112.07804}, 2021.

\bibitem[Xu et~al.(2024)Xu, Geffner, Kreis, Nie, Xu, Leskovec, Ermon, and Vahdat]{xu2024energy}
Minkai Xu, Tomas Geffner, Karsten Kreis, Weili Nie, Yilun Xu, Jure Leskovec, Stefano Ermon, and Arash Vahdat.
\newblock Energy-based diffusion language models for text generation.
\newblock \emph{arXiv preprint arXiv:2410.21357}, 2024.

\bibitem[Ye et~al.(2025)Ye, Xie, Zheng, Gao, Wu, Jiang, Li, and Kong]{dream2025}
Jiacheng Ye, Zhihui Xie, Lin Zheng, Jiahui Gao, Zirui Wu, Xin Jiang, Zhenguo Li, and Lingpeng Kong.
\newblock Dream 7b, 2025.
\newblock URL \url{https://hkunlp.github.io/blog/2025/dream}.

\bibitem[Zhang et~al.(2024)Zhang, Zeng, Wang, and Lu]{zhang2024tinyllamaopensourcesmalllanguage}
Peiyuan Zhang, Guangtao Zeng, Tianduo Wang, and Wei Lu.
\newblock Tinyllama: An open-source small language model, 2024.
\newblock URL \url{https://arxiv.org/abs/2401.02385}.

\bibitem[Zheng et~al.(2024)Zheng, Chen, Mao, Liu, Zhu, and Zhang]{zheng2024masked}
Kaiwen Zheng, Yongxin Chen, Hanzi Mao, Ming-Yu Liu, Jun Zhu, and Qinsheng Zhang.
\newblock Masked diffusion models are secretly time-agnostic masked models and exploit inaccurate categorical sampling.
\newblock \emph{arXiv preprint arXiv:2409.02908}, 2024.

\bibitem[Ziegler and Rush(2019)]{iaf}
Zachary Ziegler and Alexander Rush.
\newblock Latent normalizing flows for discrete sequences.
\newblock In \emph{International Conference on Machine Learning}, pages 7673--7682. PMLR, 2019.

\end{thebibliography}
\appendix
\newpage
\section*{NeurIPS Paper Checklist}

The checklist is designed to encourage best practices for responsible machine learning research, addressing issues of reproducibility, transparency, research ethics, and societal impact. Do not remove the checklist: {\bf The papers not including the checklist will be desk rejected.} The checklist should follow the references and follow the (optional) supplemental material.  The checklist does NOT count towards the page
limit. 

Please read the checklist guidelines carefully for information on how to answer these questions. For each question in the checklist:
\begin{itemize}
    \item You should answer \answerYes{}, \answerNo{}, or \answerNA{}.
    \item \answerNA{} means either that the question is Not Applicable for that particular paper or the relevant information is Not Available.
    \item Please provide a short (1–2 sentence) justification right after your answer (even for NA). 
\end{itemize}

{\bf The checklist answers are an integral part of your paper submission.} They are visible to the reviewers, area chairs, senior area chairs, and ethics reviewers. You will be asked to also include it (after eventual revisions) with the final version of your paper, and its final version will be published with the paper.

The reviewers of your paper will be asked to use the checklist as one of the factors in their evaluation. While "\answerYes{}" is generally preferable to "\answerNo{}", it is perfectly acceptable to answer "\answerNo{}" provided a proper justification is given (e.g., "error bars are not reported because it would be too computationally expensive" or "we were unable to find the license for the dataset we used"). In general, answering "\answerNo{}" or "\answerNA{}" is not grounds for rejection. While the questions are phrased in a binary way, we acknowledge that the true answer is often more nuanced, so please just use your best judgment and write a justification to elaborate. All supporting evidence can appear either in the main paper or the supplemental material, provided in appendix. If you answer \answerYes{} to a question, in the justification please point to the section(s) where related material for the question can be found.

IMPORTANT, please:
\begin{itemize}
    \item {\bf Delete this instruction block, but keep the section heading ``NeurIPS Paper Checklist"},
    \item  {\bf Keep the checklist subsection headings, questions/answers and guidelines below.}
    \item {\bf Do not modify the questions and only use the provided macros for your answers}.
\end{itemize}


\begin{enumerate}

\item {\bf Claims}
    \item[] Question: Do the main claims made in the abstract and introduction accurately reflect the paper's contributions and scope?
    \item[] Answer: \answerYes{} 
    \item[] Justification: The claims accuraltely reflects the paper's contributions and scope and are well supported by the experiments.
    \item[] Guidelines:
    \begin{itemize}
        \item The answer NA means that the abstract and introduction do not include the claims made in the paper.
        \item The abstract and/or introduction should clearly state the claims made, including the contributions made in the paper and important assumptions and limitations. A No or NA answer to this question will not be perceived well by the reviewers. 
        \item The claims made should match theoretical and experimental results, and reflect how much the results can be expected to generalize to other settings. 
        \item It is fine to include aspirational goals as motivation as long as it is clear that these goals are not attained by the paper. 
    \end{itemize}

\item {\bf Limitations}
    \item[] Question: Does the paper discuss the limitations of the work performed by the authors?
    \item[] Answer: \answerYes{} 
    \item[] Justification: The paper discusses its limitations, especially on inference speeds.
    \item[] Guidelines:
    \begin{itemize}
        \item The answer NA means that the paper has no limitation while the answer No means that the paper has limitations, but those are not discussed in the paper. 
        \item The authors are encouraged to create a separate "Limitations" section in their paper.
        \item The paper should point out any strong assumptions and how robust the results are to violations of these assumptions (e.g., independence assumptions, noiseless settings, model well-specification, asymptotic approximations only holding locally). The authors should reflect on how these assumptions might be violated in practice and what the implications would be.
        \item The authors should reflect on the scope of the claims made, e.g., if the approach was only tested on a few datasets or with a few runs. In general, empirical results often depend on implicit assumptions, which should be articulated.
        \item The authors should reflect on the factors that influence the performance of the approach. For example, a facial recognition algorithm may perform poorly when image resolution is low or images are taken in low lighting. Or a speech-to-text system might not be used reliably to provide closed captions for online lectures because it fails to handle technical jargon.
        \item The authors should discuss the computational efficiency of the proposed algorithms and how they scale with dataset size.
        \item If applicable, the authors should discuss possible limitations of their approach to address problems of privacy and fairness.
        \item While the authors might fear that complete honesty about limitations might be used by reviewers as grounds for rejection, a worse outcome might be that reviewers discover limitations that aren't acknowledged in the paper. The authors should use their best judgment and recognize that individual actions in favor of transparency play an important role in developing norms that preserve the integrity of the community. Reviewers will be specifically instructed to not penalize honesty concerning limitations.
    \end{itemize}

\item {\bf Theory assumptions and proofs}
    \item[] Question: For each theoretical result, does the paper provide the full set of assumptions and a complete (and correct) proof?
    \item[] Answer: \answerYes{} 
    \item[] Justification: Theoretical claims are supported with justifications and proofs.
    \item[] Guidelines:
    \begin{itemize}
        \item The answer NA means that the paper does not include theoretical results. 
        \item All the theorems, formulas, and proofs in the paper should be numbered and cross-referenced.
        \item All assumptions should be clearly stated or referenced in the statement of any theorems.
        \item The proofs can either appear in the main paper or the supplemental material, but if they appear in the supplemental material, the authors are encouraged to provide a short proof sketch to provide intuition. 
        \item Inversely, any informal proof provided in the core of the paper should be complemented by formal proofs provided in appendix or supplemental material.
        \item Theorems and Lemmas that the proof relies upon should be properly referenced. 
    \end{itemize}

    \item {\bf Experimental result reproducibility}
    \item[] Question: Does the paper fully disclose all the information needed to reproduce the main experimental results of the paper to the extent that it affects the main claims and/or conclusions of the paper (regardless of whether the code and data are provided or not)?
    \item[] Answer: \answerYes{} 
    \item[] Justification: Experiment details are discussed in great details.
    \item[] Guidelines:
    \begin{itemize}
        \item The answer NA means that the paper does not include experiments.
        \item If the paper includes experiments, a No answer to this question will not be perceived well by the reviewers: Making the paper reproducible is important, regardless of whether the code and data are provided or not.
        \item If the contribution is a dataset and/or model, the authors should describe the steps taken to make their results reproducible or verifiable. 
        \item Depending on the contribution, reproducibility can be accomplished in various ways. For example, if the contribution is a novel architecture, describing the architecture fully might suffice, or if the contribution is a specific model and empirical evaluation, it may be necessary to either make it possible for others to replicate the model with the same dataset, or provide access to the model. In general. releasing code and data is often one good way to accomplish this, but reproducibility can also be provided via detailed instructions for how to replicate the results, access to a hosted model (e.g., in the case of a large language model), releasing of a model checkpoint, or other means that are appropriate to the research performed.
        \item While NeurIPS does not require releasing code, the conference does require all submissions to provide some reasonable avenue for reproducibility, which may depend on the nature of the contribution. For example
        \begin{enumerate}
            \item If the contribution is primarily a new algorithm, the paper should make it clear how to reproduce that algorithm.
            \item If the contribution is primarily a new model architecture, the paper should describe the architecture clearly and fully.
            \item If the contribution is a new model (e.g., a large language model), then there should either be a way to access this model for reproducing the results or a way to reproduce the model (e.g., with an open-source dataset or instructions for how to construct the dataset).
            \item We recognize that reproducibility may be tricky in some cases, in which case authors are welcome to describe the particular way they provide for reproducibility. In the case of closed-source models, it may be that access to the model is limited in some way (e.g., to registered users), but it should be possible for other researchers to have some path to reproducing or verifying the results.
        \end{enumerate}
    \end{itemize}

\item {\bf Open access to data and code}
    \item[] Question: Does the paper provide open access to the data and code, with sufficient instructions to faithfully reproduce the main experimental results, as described in supplemental material?
    \item[] Answer: \answerNo{} 
    \item[] Justification: The paper doesn't release its code at this time but great details on experiment setups are provided and will release its code at an appropriate time.
    \item[] Guidelines:
    \begin{itemize}
        \item The answer NA means that paper does not include experiments requiring code.
        \item Please see the NeurIPS code and data submission guidelines (\url{https://nips.cc/public/guides/CodeSubmissionPolicy}) for more details.
        \item While we encourage the release of code and data, we understand that this might not be possible, so “No” is an acceptable answer. Papers cannot be rejected simply for not including code, unless this is central to the contribution (e.g., for a new open-source benchmark).
        \item The instructions should contain the exact command and environment needed to run to reproduce the results. See the NeurIPS code and data submission guidelines (\url{https://nips.cc/public/guides/CodeSubmissionPolicy}) for more details.
        \item The authors should provide instructions on data access and preparation, including how to access the raw data, preprocessed data, intermediate data, and generated data, etc.
        \item The authors should provide scripts to reproduce all experimental results for the new proposed method and baselines. If only a subset of experiments are reproducible, they should state which ones are omitted from the script and why.
        \item At submission time, to preserve anonymity, the authors should release anonymized versions (if applicable).
        \item Providing as much information as possible in supplemental material (appended to the paper) is recommended, but including URLs to data and code is permitted.
    \end{itemize}

\item {\bf Experimental setting/details}
    \item[] Question: Does the paper specify all the training and test details (e.g., data splits, hyperparameters, how they were chosen, type of optimizer, etc.) necessary to understand the results?
    \item[] Answer: \answerYes{} 
    \item[] Justification: Experiment details are specified.
    \item[] Guidelines:
    \begin{itemize}
        \item The answer NA means that the paper does not include experiments.
        \item The experimental setting should be presented in the core of the paper to a level of detail that is necessary to appreciate the results and make sense of them.
        \item The full details can be provided either with the code, in appendix, or as supplemental material.
    \end{itemize}

\item {\bf Experiment statistical significance}
    \item[] Question: Does the paper report error bars suitably and correctly defined or other appropriate information about the statistical significance of the experiments?
    \item[] Answer: \answerNo{} 
    \item[] Justification: Error bars are not reported because it would be too computationally
expensive
    \item[] Guidelines:
    \begin{itemize}
        \item The answer NA means that the paper does not include experiments.
        \item The authors should answer "Yes" if the results are accompanied by error bars, confidence intervals, or statistical significance tests, at least for the experiments that support the main claims of the paper.
        \item The factors of variability that the error bars are capturing should be clearly stated (for example, train/test split, initialization, random drawing of some parameter, or overall run with given experimental conditions).
        \item The method for calculating the error bars should be explained (closed form formula, call to a library function, bootstrap, etc.)
        \item The assumptions made should be given (e.g., Normally distributed errors).
        \item It should be clear whether the error bar is the standard deviation or the standard error of the mean.
        \item It is OK to report 1-sigma error bars, but one should state it. The authors should preferably report a 2-sigma error bar than state that they have a 96\% CI, if the hypothesis of Normality of errors is not verified.
        \item For asymmetric distributions, the authors should be careful not to show in tables or figures symmetric error bars that would yield results that are out of range (e.g. negative error rates).
        \item If error bars are reported in tables or plots, The authors should explain in the text how they were calculated and reference the corresponding figures or tables in the text.
    \end{itemize}

\item {\bf Experiments compute resources}
    \item[] Question: For each experiment, does the paper provide sufficient information on the computer resources (type of compute workers, memory, time of execution) needed to reproduce the experiments?
    \item[] Answer: \answerYes{} 
    \item[] Justification: Experiment compute resourses are listed.
    \item[] Guidelines:
    \begin{itemize}
        \item The answer NA means that the paper does not include experiments.
        \item The paper should indicate the type of compute workers CPU or GPU, internal cluster, or cloud provider, including relevant memory and storage.
        \item The paper should provide the amount of compute required for each of the individual experimental runs as well as estimate the total compute. 
        \item The paper should disclose whether the full research project required more compute than the experiments reported in the paper (e.g., preliminary or failed experiments that didn't make it into the paper). 
    \end{itemize}
    
\item {\bf Code of ethics}
    \item[] Question: Does the research conducted in the paper conform, in every respect, with the NeurIPS Code of Ethics \url{https://neurips.cc/public/EthicsGuidelines}?
    \item[] Answer: \answerYes{} 
    \item[] Justification: This research conform, in every respect, with the NeurIPS Code of Ethics
    \item[] Guidelines:
    \begin{itemize}
        \item The answer NA means that the authors have not reviewed the NeurIPS Code of Ethics.
        \item If the authors answer No, they should explain the special circumstances that require a deviation from the Code of Ethics.
        \item The authors should make sure to preserve anonymity (e.g., if there is a special consideration due to laws or regulations in their jurisdiction).
    \end{itemize}

\item {\bf Broader impacts}
    \item[] Question: Does the paper discuss both potential positive societal impacts and negative societal impacts of the work performed?
    \item[] Answer: \answerYes{} 
    \item[] Justification: Social impacts are discussed in this paper.
    \item[] Guidelines:
    \begin{itemize}
        \item The answer NA means that there is no societal impact of the work performed.
        \item If the authors answer NA or No, they should explain why their work has no societal impact or why the paper does not address societal impact.
        \item Examples of negative societal impacts include potential malicious or unintended uses (e.g., disinformation, generating fake profiles, surveillance), fairness considerations (e.g., deployment of technologies that could make decisions that unfairly impact specific groups), privacy considerations, and security considerations.
        \item The conference expects that many papers will be foundational research and not tied to particular applications, let alone deployments. However, if there is a direct path to any negative applications, the authors should point it out. For example, it is legitimate to point out that an improvement in the quality of generative models could be used to generate deepfakes for disinformation. On the other hand, it is not needed to point out that a generic algorithm for optimizing neural networks could enable people to train models that generate Deepfakes faster.
        \item The authors should consider possible harms that could arise when the technology is being used as intended and functioning correctly, harms that could arise when the technology is being used as intended but gives incorrect results, and harms following from (intentional or unintentional) misuse of the technology.
        \item If there are negative societal impacts, the authors could also discuss possible mitigation strategies (e.g., gated release of models, providing defenses in addition to attacks, mechanisms for monitoring misuse, mechanisms to monitor how a system learns from feedback over time, improving the efficiency and accessibility of ML).
    \end{itemize}
    
\item {\bf Safeguards}
    \item[] Question: Does the paper describe safeguards that have been put in place for responsible release of data or models that have a high risk for misuse (e.g., pretrained language models, image generators, or scraped datasets)?
    \item[] Answer: \answerNA{} 
    \item[] Justification: This paper poses no such risks.
    \item[] Guidelines:
    \begin{itemize}
        \item The answer NA means that the paper poses no such risks.
        \item Released models that have a high risk for misuse or dual-use should be released with necessary safeguards to allow for controlled use of the model, for example by requiring that users adhere to usage guidelines or restrictions to access the model or implementing safety filters. 
        \item Datasets that have been scraped from the Internet could pose safety risks. The authors should describe how they avoided releasing unsafe images.
        \item We recognize that providing effective safeguards is challenging, and many papers do not require this, but we encourage authors to take this into account and make a best faith effort.
    \end{itemize}

\item {\bf Licenses for existing assets}
    \item[] Question: Are the creators or original owners of assets (e.g., code, data, models), used in the paper, properly credited and are the license and terms of use explicitly mentioned and properly respected?
    \item[] Answer: \answerYes{} 
    \item[] Justification: Owners of assets used are properly credited in the paper and any license is respected.
    \item[] Guidelines:
    \begin{itemize}
        \item The answer NA means that the paper does not use existing assets.
        \item The authors should cite the original paper that produced the code package or dataset.
        \item The authors should state which version of the asset is used and, if possible, include a URL.
        \item The name of the license (e.g., CC-BY 4.0) should be included for each asset.
        \item For scraped data from a particular source (e.g., website), the copyright and terms of service of that source should be provided.
        \item If assets are released, the license, copyright information, and terms of use in the package should be provided. For popular datasets, \url{paperswithcode.com/datasets} has curated licenses for some datasets. Their licensing guide can help determine the license of a dataset.
        \item For existing datasets that are re-packaged, both the original license and the license of the derived asset (if it has changed) should be provided.
        \item If this information is not available online, the authors are encouraged to reach out to the asset's creators.
    \end{itemize}

\item {\bf New assets}
    \item[] Question: Are new assets introduced in the paper well documented and is the documentation provided alongside the assets?
    \item[] Answer: \answerNA{} 
    \item[] Justification: The paper does not release new assets
    \item[] Guidelines:
    \begin{itemize}
        \item The answer NA means that the paper does not release new assets.
        \item Researchers should communicate the details of the dataset/code/model as part of their submissions via structured templates. This includes details about training, license, limitations, etc. 
        \item The paper should discuss whether and how consent was obtained from people whose asset is used.
        \item At submission time, remember to anonymize your assets (if applicable). You can either create an anonymized URL or include an anonymized zip file.
    \end{itemize}

\item {\bf Crowdsourcing and research with human subjects}
    \item[] Question: For crowdsourcing experiments and research with human subjects, does the paper include the full text of instructions given to participants and screenshots, if applicable, as well as details about compensation (if any)? 
    \item[] Answer: \answerNA{} 
    \item[] Justification: The paper does not involve crowdsourcing nor research with human subjects
    \item[] Guidelines:
    \begin{itemize}
        \item The answer NA means that the paper does not involve crowdsourcing nor research with human subjects.
        \item Including this information in the supplemental material is fine, but if the main contribution of the paper involves human subjects, then as much detail as possible should be included in the main paper. 
        \item According to the NeurIPS Code of Ethics, workers involved in data collection, curation, or other labor should be paid at least the minimum wage in the country of the data collector. 
    \end{itemize}

\item {\bf Institutional review board (IRB) approvals or equivalent for research with human subjects}
    \item[] Question: Does the paper describe potential risks incurred by study participants, whether such risks were disclosed to the subjects, and whether Institutional Review Board (IRB) approvals (or an equivalent approval/review based on the requirements of your country or institution) were obtained?
    \item[] Answer: \answerNA{} 
    \item[] Justification: The paper does not involve crowdsourcing nor research with human subjects
    \item[] Guidelines:
    \begin{itemize}
        \item The answer NA means that the paper does not involve crowdsourcing nor research with human subjects.
        \item Depending on the country in which research is conducted, IRB approval (or equivalent) may be required for any human subjects research. If you obtained IRB approval, you should clearly state this in the paper. 
        \item We recognize that the procedures for this may vary significantly between institutions and locations, and we expect authors to adhere to the NeurIPS Code of Ethics and the guidelines for their institution. 
        \item For initial submissions, do not include any information that would break anonymity (if applicable), such as the institution conducting the review.
    \end{itemize}

\item {\bf Declaration of LLM usage}
    \item[] Question: Does the paper describe the usage of LLMs if it is an important, original, or non-standard component of the core methods in this research? Note that if the LLM is used only for writing, editing, or formatting purposes and does not impact the core methodology, scientific rigorousness, or originality of the research, declaration is not required.
    \item[] Answer: \answerNA{} 
    \item[] Justification: LLMs are not involved as a core component of the formulation of this research.
    \item[] Guidelines:
    \begin{itemize}
        \item The answer NA means that the core method development in this research does not involve LLMs as any important, original, or non-standard components.
        \item Please refer to our LLM policy (\url{https://neurips.cc/Conferences/2025/LLM}) for what should or should not be described.
    \end{itemize}

\end{enumerate}   

\newpage
\renewcommand{\thesection}{\Alph{section}}
\newpage
\setcounter{section}{0}
\section{Proof}
\label{sec:proof}

Throughout this appendix we work with a \emph{single--token} random variable that
takes values in the augmented vocabulary $\mathcal{V}_{+}:=\mathcal{V}\cup\{\texttt{[MASK]}\}$.  The distinguished absorbing symbol is denoted by $m\in\mathcal{V}_{+}$ and we write $e_{m}$ for the corresponding standard basis vector.

Let $\mathbf{x}_{0}\sim p_{0}$ be drawn from the empirical data distribution, which assigns zero probability to $m$.  Fix a diffusion horizon $T\in\mathbb{N}$.

\subsection{Proof of Proposition~\ref{prop:correspondence}}
\label{app:NM_vs_Markov}

\textit{An absorbing--state non--Markovian discrete diffusion process with marginal transition kernel $\bar{\mathbf{Q}}_{t}^{}=(1-\alpha_{t})\mathbf{I}+\alpha_{t}\mathbf{1}e_{m}^{\top}$ admits a bijection to an absorbing--state Markovian discrete diffusion process with marginal transition kernel $\bar{\mathbf{Q}}_{t}^{*}=(1-\alpha_{t}^{*})\mathbf{I}+\alpha_{t}^{*}\mathbf{1}e_{m}^{\top}$ such that the two processes exhibit identical mutual--information decay between $\mathbf{x}_{t:T}$ and $\mathbf{x}_{0}$, provided that the coefficients satisfy $\alpha_{t}^{*}=\prod_{\tau =t}^{T}\alpha_{\tau}$.}\par\medskip

We begin by formalising the two forward corruption processes that
will be compared.
\subsubsection{Forward–process definitions}
\begin{definition}[Markovian absorbing diffusion]
\label{def:markov}
Let $\boldsymbol{\beta}=(\beta_{1},\dots,\beta_{T})\subset[0,1]^{T}$ be the \emph{single--step masking probabilities}. The Markov chain $(\mathbf{x}_{s})_{s=0}^{T}$ is specified by
\begin{equation}
\label{eq:Markov_forward}
\mathbf{x}_{s}\mid \mathbf{x}_{s-1} \sim (1-\beta_{s})\delta_{\mathbf{x}_{s-1}}+\beta_{s}\delta_{m}, \qquad s=1,\dots,T.
\end{equation}
Because $m$ is absorbing, the cumulative masking probability after $t$ steps is
\begin{equation}
\label{eq:alpha_beta}
\alpha_{t}^{*}:=\Pr\bigl(\mathbf{x}_{t}=m\mid \mathbf{x}_{0}\bigr)=1-\prod_{s=1}^{t}(1-\beta_{s}), \qquad t=1,\dots,T,
\end{equation}
which yields the marginal transition kernel
\begin{equation*}
\bar{\mathbf{Q}}_{t}^{\mathrm{M}}=(1-\alpha_{t}^{*})\mathbf{I}+\alpha_{t}^{*}\mathbf{1}e_{m}^{\top}.
\end{equation*}
\end{definition}

\begin{definition}[Non--Markovian absorbing diffusion]
\label{def:non_markov}
For a \emph{noise schedule} $\boldsymbol{\alpha}=(\alpha_{1},\dots,\alpha_{T})\subset[0,1]^{T}$ we generate the forward trajectory by
\begin{equation}
\label{eq:NM_forward}
\mathbf{x}_{s}\mid \mathbf{x}_{0} \overset{\text{i.i.d.}}{\sim} (1-\alpha_{s})\delta_{\mathbf{x}_{0}}+\alpha_{s}\delta_{m}, \qquad s=1,\dots,T.
\end{equation}
Each $\mathbf{x}_{s}$ is obtained by independently masking $\mathbf{x}_{0}$ with probability $\alpha_{s}$. The associated marginal transition kernel is
\begin{equation*}
\bar{\mathbf{Q}}_{t}^{\mathrm{NM}} = (1-\alpha_{t})\mathbf{I}+\alpha_{t}\mathbf{1}e_{m}^{\top}, \qquad t=1,\dots,T.
\end{equation*}
\end{definition}

\subsubsection{Mutual--information computations and decay measure}

For any pair of random variables $(U,V)$ we recall that $I(U;V) = H(U) - H(U \mid V)$. The entropy $H(\mathbf{x}_0) = -\sum_{v \in \mathcal{V}} p_0(v) \log p_0(v)$ is finite and strictly positive.

To make comparisons between Markovian and non-Markovian diffusion processes more interpretable, we introduce the \textbf{normalized mutual information decay} following~\citep{d3pm}:
\begin{equation}
D_t:=1-\frac{I\left(\mathbf{x}_t ; \mathbf{x}_0\right)}{H\left(\mathbf{x}_0\right)}=\frac{H\left(\mathbf{x}_0, \mathbf{x}_t\right)-H\left(\mathbf{x}_t\right)}{H\left(\mathbf{x}_0\right)}
\end{equation}
which quantifies the proportion of information about $\mathbf{x}_0$ that is lost after corruption step $t$.

\begin{lemma}[Markovian suffix information]
\label{lem:Markov_MI}
For the process of Definition~\ref{def:markov} and every $t\in\{1,\dots,T\}$,
\begin{equation}
I_{\mathrm{M}}\left(\mathbf{x}_{t: T} ; \mathbf{x}_0\right)=H\left(\mathbf{x}_0\right)\left(1-\alpha_t^*\right), \quad D_t^{\mathrm{M}}=\alpha_t^*
\end{equation}
\end{lemma}
Here we give two proofs:
\begin{proof}[Proof 1 (Direct Expansion)]
Because the chain is absorbing, the suffix \(\mathbf{x}_{t:T}\) is a deterministic
function of the single variable \(\mathbf{x}_{t}\); hence
\[
I_{\mathrm{M}}\!\bigl(\mathbf{x}_{t:T};\mathbf{x}_0\bigr)
   \;=\;I\!\bigl(\mathbf{x}_{t};\mathbf{x}_0\bigr).
\]

There are two mutually exclusive outcomes:

\begin{enumerate}[label=(\roman*), leftmargin=2em, labelsep=0.5em]
\item {Un‑masked:}  
      with probability \(1-\alpha_t^{*}\) the token is uncorrupted and
      \(\mathbf{x}_{t}=\mathbf{x}_0\).

\item {Masked:}  
      with probability \(\alpha_t^{*}\) the token is replaced by the absorbing
      symbol \(e_m\), which is independent of \(\mathbf{x}_0\).
\end{enumerate}

Writing \(p_0(v)=\Pr(\mathbf{x}_0=v)\), the joint pmf is
\[
p(\mathbf{x}_0=v,\mathbf{x}_t=u)=
\begin{cases}
(1-\alpha_t^{*})\,p_0(v), & u=v\neq e_m,\\[6pt]
\alpha_t^{*}\,p_0(v),     & u=e_m,\\[6pt]
0,                        & \text{otherwise}.
\end{cases}
\]
The corresponding marginal of \(\mathbf{x}_t\) is
\[
p(\mathbf{x}_t=u)=
\begin{cases}
(1-\alpha_t^{*})\,p_0(u), & u\neq e_m,\\[6pt]
\alpha_t^{*},             & u=e_m.
\end{cases}
\]

{Direct expansion of mutual information.}
\[
I\!\bigl(\mathbf{x}_{t};\mathbf{x}_0\bigr)
   \;=\;\sum_{u,v}p(u,v)\log\frac{p(u,v)}{p(u)\,p(v)}.
\]

Only two cases have non‑zero probability:

\begin{enumerate}[label=(\roman*), leftmargin=2em, labelsep=0.5em]
\item \emph{Case \(u=v\neq e_m\).}\;
      Here \(p(u,v)=(1-\alpha_t^{*})p_0(v)\) and
      \(p(u)=(1-\alpha_t^{*})p_0(v)\), so
      \[
        \log\frac{p(u,v)}{p(u)p(v)}
        =\log\frac{(1-\alpha_t^{*})p_0(v)}{(1-\alpha_t^{*})p_0(v)\,p_0(v)}
        =-\log p_0(v).
      \]
      The total contribution is
      \[
        \sum_{v\neq e_m}(1-\alpha_t^{*})p_0(v)\bigl[-\log p_0(v)\bigr]
        =(1-\alpha_t^{*})H(\mathbf{x}_0).
      \]

\item \emph{Case \(u=e_m\).}\;
      In this case \(p(u,v)=\alpha_t^{*}p_0(v)\) and \(p(u)=\alpha_t^{*}\); the
      logarithmic term vanishes, so the contribution is zero.
\end{enumerate}

Hence
\[
I\!\bigl(\mathbf{x}_{t};\mathbf{x}_0\bigr)=(1-\alpha_t^{*})H(\mathbf{x}_0).
\]

\[
D_t^{\mathrm{M}}
   =1-\frac{I\!\left(\mathbf{x}_{t};\mathbf{x}_0\right)}{H(\mathbf{x}_0)}
   =1-(1-\alpha_t^{*})
   =\alpha_t^{*}.
\]

Since \(I(\mathbf{x}_{t:T};\mathbf{x}_0)=I(\mathbf{x}_{t};\mathbf{x}_0)\), the
stated identities follow.
\end{proof}
\begin{proof}[Proof 2 (Conditional Entropy)]
Condition on the observed token $\mathbf{x}_t$:

\begin{enumerate}[label=(\roman*), leftmargin=2em, labelsep=0.5em]
\item If $\mathbf{x}_t = e_m$ (probability $\alpha_t^{*}$), masking reveals nothing,
      so $H(\mathbf{x}_0 \mid \mathbf{x}_t = e_m)=H(\mathbf{x}_0)$.
\item If $\mathbf{x}_t = v \neq e_m$ (probability $1-\alpha_t^{*}$),
      we know $\mathbf{x}_0=v$ exactly, hence $H(\mathbf{x}_0 \mid \mathbf{x}_t=v)=0$.
\end{enumerate}

Averaging,
\[
H(\mathbf{x}_0\mid\mathbf{x}_t)=
\alpha_t^{*}\,H(\mathbf{x}_0)+(1-\alpha_t^{*})\cdot 0
=\alpha_t^{*}\,H(\mathbf{x}_0).
\]

Therefore
\[
I(\mathbf{x}_t;\mathbf{x}_0)
=H(\mathbf{x}_0)-H(\mathbf{x}_0\mid\mathbf{x}_t)
=(1-\alpha_t^{*})\,H(\mathbf{x}_0),
\qquad
D_t^{\mathrm M}=1-\frac{I(\mathbf{x}_t;\mathbf{x}_0)}{H(\mathbf{x}_0)}
=\alpha_t^{*}.
\]

Because $\mathbf{x}_{t:T}$ is a deterministic function of $\mathbf{x}_t$,
the same formula holds for $I(\mathbf{x}_{t:T};\mathbf{x}_0)$.
\end{proof}

\begin{lemma}[Non--Markovian suffix information]
\label{lem:NM_MI}
For the process of Definition~\ref{def:non_markov} and every $t\in\{1,\dots,T\}$,
\begin{equation}
I_{\mathrm{NM}}\left(\mathbf{x}_{t: T} ; \mathbf{x}_0\right)=H\left(\mathbf{x}_0\right)\left(1-\prod_{\tau=t}^T \alpha_\tau\right), \quad D_t^{\mathrm{NM}}=\prod_{\tau=t}^T \alpha_\tau
\end{equation}
\end{lemma}
\begin{proof}
Define the indicator
\(Z:=\mathbf 1\{\mathbf{x}_{\tau}=e_m \text{ for every }\tau=t,\dots ,T\}\).
Conditioned on \(\mathbf{x}_0\), each coordinate is masked
independently with probability \(\alpha_\tau\), so
\[
\Pr(Z=1)=\prod_{\tau=t}^{T}\alpha_\tau ,\qquad
\Pr(Z=0)=1-\prod_{\tau=t}^{T}\alpha_\tau .
\]
The event \(Z\) depends only on the masking coins,
hence is independent of the value of \(\mathbf{x}_0\).

\smallskip
\noindent For {conditional entropy \(H(\mathbf{x}_0\mid\mathbf{x}_{t:T})\).}
\begin{enumerate}[label=(\roman*), leftmargin=2em, labelsep=0.5em]
\item If \(Z=1\) (all tokens masked, probability \(\prod\alpha_\tau\)),
      the suffix reveals nothing, so
      \(H(\mathbf{x}_0\mid\mathbf{x}_{t:T})=H(\mathbf{x}_0)\).
\item If \(Z=0\) (at least one un‑masked coordinate, probability
      \(1-\prod\alpha_\tau\)), every un‑masked coordinate equals
      \(\mathbf{x}_0\); thus \(\mathbf{x}_0\) is known exactly and
      \(H(\mathbf{x}_0\mid\mathbf{x}_{t:T})=0\).
\end{enumerate}
Averaging over \(Z\),
\[
H(\mathbf{x}_0\mid\mathbf{x}_{t:T})
      =\Bigl(\prod_{\tau=t}^{T}\alpha_\tau\Bigr)H(\mathbf{x}_0).
\]

\smallskip
Using \(I(U;V)=H(U)-H(U\mid V)\),
\[
I_{\mathrm{NM}}\!\bigl(\mathbf{x}_{t:T};\mathbf{x}_0\bigr)
   =H(\mathbf{x}_0)-H(\mathbf{x}_0\mid\mathbf{x}_{t:T})
   =H(\mathbf{x}_0)\Bigl(1-\prod_{\tau=t}^{T}\alpha_\tau\Bigr).
\]

\smallskip
\[
D_t^{\mathrm{NM}}
   =1-\frac{I_{\mathrm{NM}}\!\left(\mathbf{x}_{t:T};\mathbf{x}_0\right)}
           {H(\mathbf{x}_0)}
   =\prod_{\tau=t}^{T}\alpha_\tau.
\]
\end{proof}

These results show that both the Markovian and non-Markovian processes exhibit the same normalized mutual information decay if their cumulative masking curves satisfy:
\begin{equation}
\alpha_t^*=\prod_{\tau=t}^T \alpha_\tau.
\end{equation}

Since \(\alpha^*_t\) and \(\beta_t\) have correspondence as shown in Equation~\ref{eq:alpha_beta}, one can easily further derive:
\begin{proposition}[Equivalence in mutual-information decay]
\label{prop:equivalence}
Let $\boldsymbol{\alpha} \subset [0,1]^T$ be a non--Markovian schedule and define the effective cumulative schedule \(\alpha_t^* := \prod_{\tau=t}^T \alpha_\tau\). Then there exists a Markovian schedule \(\boldsymbol{\beta} \subset [0,1]^T\) defined by
\begin{equation}
\beta_t := 1 - \frac{1 - \alpha_t^*}{1 - \alpha_{t-1}^*}, \qquad \alpha_0^* := 0,
\end{equation}
such that:
\begin{equation}
D_t^{\mathrm{NM}} = D_t^{\mathrm{M}} \quad \text{and} \quad I_{\mathrm{NM}}\left(\mathbf{x}_{t: T} ; \mathbf{x}_0\right) = I_{\mathrm{M}}\left(\mathbf{x}_{t: T} ; \mathbf{x}_0\right),
\end{equation}
for all \(t = 1, \dots, T\). This ensures equivalence in both absolute and relative information loss.
\end{proposition}

\paragraph{Matching the linear marginal used in prior work.}
Most existing studies on absorbing‑state Markovian diffusions~\cite{d3pm, lou2024discretediffusionmodelingestimating, mdlm} adopt the \emph{linear} cumulative–masking
curve
\(
  \alpha_{t}^{\mathrm{lin},*}=t/T,\; t=0,\dots,T.
\)
For a fair comparison by default we employ a non‑Markovian \emph{independent} schedule
${\alpha}^{\mathrm{lin}}$ that reproduces exactly the same
marginals in our experiments, i.e.\ satisfies
\(
  \alpha_{t}^{\mathrm{lin},*}
  =\prod_{\tau=t}^{T}\alpha_{\tau}^{\mathrm{lin}}.
\)
Using the backward recursion
$\alpha_{t}= \alpha_{t}^{*}/\alpha_{t+1}^{*}$ (Sec.~\ref{app:NM_vs_Markov})
gives the closed‑form
\[
  {\;
    \alpha_{t}^{\mathrm{lin}}
    =\frac{t}{t+1},\quad t=1,\dots,T-1,\qquad
    \alpha_{T}^{\mathrm{lin}}=1
  \;}
\]

\subsection{Derivation of the Non-Markovian Evidence Lower Bound (ELBO)}
\label{app:elbo_proof}
We derive a variational lower bound on the marginal log-likelihood of the observed data $\mathbf{x}_0$ under the non-Markovian discrete diffusion model.


We start with the marginal likelihood of the observed data:

$$
\log p_\theta\left(\mathbf{x}_0\right)=\log \int p_\theta\left(\mathbf{x}_0, \mathbf{x}_{1: T}\right) d \mathbf{x}_{1: T}=\log \int \frac{p_\theta\left(\mathbf{x}_0, \mathbf{x}_{1: T}\right)}{q\left(\mathbf{x}_{1: T} \mid \mathbf{x}_0\right)} q\left(\mathbf{x}_{1: T} \mid \mathbf{x}_0\right) d \mathbf{x}_{1: T}
$$

Applying Jensen's inequality yields the Evidence Lower Bound (ELBO):

$$
\log p_\theta\left(\mathbf{x}_0\right) \geq \mathbb{E}_{q\left(\mathbf{x}_{1: T} \mid \mathbf{x}_0\right)}\left[\log \frac{p_\theta\left(\mathbf{x}_0, \mathbf{x}_{1: T}\right)}{q\left(\mathbf{x}_{1: T} \mid \mathbf{x}_0\right)}\right]=\mathcal{L}_{\text {non-markov }}
$$


In the non-Markovian case, the joint distribution over the reverse generative process is:

$$
p_\theta\left(\mathbf{x}_0, \mathbf{x}_{1: T}\right)=p_\theta\left(\mathbf{x}_T\right) \prod_{t=1}^T p_\theta\left(\mathbf{x}_{t-1} \mid \mathbf{x}_{t: T}\right)
$$

where $\mathbf{x}_0$ is generated at the final step.
The approximate posterior, due to the independent corruption assumption, factorizes as:

$$
q\left(\mathbf{x}_{1: T} \mid \mathbf{x}_0\right)=\prod_{t=1}^T q\left(\mathbf{x}_t \mid \mathbf{x}_0\right)
$$

Finally, we have the ELBO expansion:
\begin{equation}
\begin{aligned}
\mathcal{L}_{\text{non-markov}} =\; &
  \mathbb{E}_{q(\mathbf{x}_{1:T} \mid \mathbf{x}_0)} \Bigg[
     \log \frac{p_\theta(\mathbf{x}_T)}{q(\mathbf{x}_T \mid \mathbf{x}_0)}
     + \sum_{t=2}^{T}
         \log \frac{p_\theta(\mathbf{x}_{t-1} \mid \mathbf{x}_{t:T})}
                    {q(\mathbf{x}_{t-1} \mid \mathbf{x}_0)}
     + \log p_\theta(\mathbf{x}_0 \mid \mathbf{x}_{1:T})
  \Bigg] \\
=\; & \underbrace{
       \mathbb{E}_{q(\mathbf{x}_{1:T} \mid \mathbf{x}_0)}
       \log p_\theta(\mathbf{x}_0 \mid \mathbf{x}_{1:T})
     }_{\text{Reconstruction}} 
     - \underbrace{
       \operatorname{KL}\left(q(\mathbf{x}_T \mid \mathbf{x}_0) \;\|\; p_\theta(\mathbf{x}_T)\right)
     }_{\text{Prior KL}} \\
  & - \underbrace{
       \sum_{t=2}^{T}
       \mathbb{E}_{q(\mathbf{x}_{t:T} \mid \mathbf{x}_0)}
       \operatorname{KL}\left(q(\mathbf{x}_{t-1} \mid \mathbf{x}_0) \;\|\; p_\theta(\mathbf{x}_{t-1} \mid \mathbf{x}_{t:T})\right)
     }_{\text{Reverse KLs}}
\end{aligned}
\label{eq:nonmarkov-elbo}
\end{equation}
We can denote the accumulated reverse KL terms as:
\begin{equation}
    \mathcal{L}_T=\sum_{t=2}^T \mathbb{E}_{q\left(\mathbf{x}_{t: T} \mid \mathbf{x}_0\right)} \operatorname{KL}\left(q\left(\mathbf{x}_{t-1} \mid \mathbf{x}_0\right) \| p_\theta\left(\mathbf{x}_{t-1} \mid \mathbf{x}_{t: T}\right)\right.
\end{equation}

With this notation, the non-Markovian ELBO simplifies to Equation~\ref{eq:elbo_loss}:
\begin{equation}
    \mathcal{L}_{\text {non-markov }}=\mathbb{E}_{q\left(\mathbf{x}_{1: T} \mid \mathbf{x}_0\right)} \log p_\theta\left(\mathbf{x}_0 \mid \mathbf{x}_{1: T}\right)-\operatorname{KL}\left(q\left(\mathbf{x}_T \mid \mathbf{x}_0\right) \| p_\theta\left(\mathbf{x}_T\right)\right)-\mathcal{L}_T
\end{equation}
Note that the second term, corresponding to the prior KL, is constant for most diffusion kernels and can be omitted during training.

\subsection{Proof of Proposition \ref{prop:objective}}
\textit{
Suppose the non-Markovian diffusion process adopt an absorbing marginal kernel \(q\left(\boldsymbol{x}_t \mid \boldsymbol{x}_0\right) = \operatorname{Cat}\left(\mathbf{x}_t ; \mathbf{x}_0 \bar{\mathbf{Q}}_t\right)\), where \(\bar{\mathbf{Q}}_t = \left(1-\alpha_t\right) \mathbf{I}+\alpha_t \mathbf{1} e_m^{\top}\) and \(\alpha_t\) is a increasing function with \(\alpha_0 \approx 0\) and \(\alpha_T \approx 1\). The ELBO loss in Equation~\eqref{eq:elbo_loss} can be further simplified to:
\begin{equation*}
    \mathcal{L}_{\text{absorb}} = 
    \mathbb{E}_{\mathbf{x}_{1: T} \sim q\bigl( \mathbf{x}_{1:T} \mid \mathbf{x}_0\bigr)}
    \sum_{t=1}^T
    \bigl[
    \alpha_{t-1} \mathbf{x}_0^{\top}\log \mu_\theta(\mathbf{x}_{t:T}, t)
    \bigr].
\end{equation*}
}
\label{app:absorb_elbo_proof}
\paragraph{Proof.} Suppose the non-Markovian diffusion process adopts an absorbing marginal kernel \(\overline{\mathbf{Q}}_t=\left(1-\alpha_t\right) \mathbf{I}+\alpha_t \mathbf{1} e_m^{\top}\), we have
\begin{equation}
    q\left(\mathbf{x}_t | \mathbf{x}_0 \right) 
    = 
    \begin{cases}
    \alpha_t & \mathbf{x}_t=e_m \\
    1 - \alpha_t & \mathbf{x}_t = \mathbf{x}_0 \\
    0 & \text{otherwise}.
    \end{cases}
\end{equation}

We now evaluate the per-step KL term in the ELBO~\ref{eq:nonmarkov-elbo}:
\begin{equation}
\begin{aligned}
\text{KL}\left(q(\mathbf{x}_{t-1} \mid \mathbf{x}_0) \,\|\, p_\theta(\mathbf{x}_{t-1} \mid \mathbf{x}_{t:T})\right)
&= \text{KL}\left(q(\mathbf{x}_{t-1} \mid \mathbf{x}_0) \,\|\, q(\mathbf{x}_{t-1} \mid \mathbf{x}_{t+1}, \mu_\theta(\mathbf{x}_{t:T}, t))\right) \\
&= \underbrace{q(\mathbf{x}_{t-1} = e_m \mid \mathbf{x}_0) \log \frac{q(\mathbf{x}_{t-1} = e_m \mid \mathbf{x}_0)}{q(\mathbf{x}_{t-1} = e_m \mid \mu_\theta(\mathbf{x}_{t:T}, t))}}_{0} \\
&\quad + \sum_{k \ne m} q(\mathbf{x}_{t-1} = e_k \mid \mathbf{x}_0) \log \frac{q(\mathbf{x}_{t-1} = e_k \mid \mathbf{x}_0)}{q(\mathbf{x}_{t-1} = e_k \mid \mu_\theta(\mathbf{x}_{t:T}, t))} \\
&= \sum_{k \ne m} \alpha_{t-1} \mathbf{x}_0^\top e_k \log \frac{\mathbf{x}_0^\top e_k}{\left[\mu_\theta(\mathbf{x}_{t:T}, t)\right]^\top e_k} \\
&= -\alpha_{t-1} \mathbf{x}_0^\top \log \mu_\theta(\mathbf{x}_{t:T}, t)
\end{aligned}
\label{sq:per_step_loss}
\end{equation}

Next, consider the reconstruction term:
\begin{equation}
\begin{aligned}
        \mathbb{E}_{q\left(\mathbf{x}_{1: T} \mid \mathbf{x}_0\right)} \log p_\theta\left(\mathbf{x}_0 \mid \mathbf{x}_{1: T}\right) &= \mathbb{E}_{q\left(\mathbf{x}_{1: T} \mid \mathbf{x}_0\right)}  \mathbf{x}_0^\top \log \mu_\theta(\mathbf{x}_{1:T}, 1) \\
        &= \mathbb{E}_{q\left(\mathbf{x}_{1: T} \mid \mathbf{x}_0\right)} \alpha_0 \mathbf{x}_0^\top \log \mu_\theta(\mathbf{x}_{1:T}, 1)
\end{aligned}
\end{equation}

Finally, the prior KL term vanishes:
\begin{equation}
    \operatorname{KL}\left(q\left(\mathbf{x}_T \mid \mathbf{x}_0\right) \| p\left(\mathbf{x}_T\right)\right)=\operatorname{KL}\left(\delta_{\mathbf{x}_T, e_m} \| \delta_{\mathbf{x}_T, e_m}\right)=0
\end{equation}
Putting all components together, we obtain the full ELBO:
\begin{equation}
\mathcal{L}_{\text {absorb }}=\mathbb{E}_{\mathbf{x}_{1: T} \sim q\left(\mathbf{x}_{1: T} \mid \mathbf{x}_0\right)} \sum_{t=1}^T\left[\alpha_{t-1} \mathbf{x}_0^{\top} \log \mu_\theta\left(\mathbf{x}_{t: T}, t\right)\right]
\end{equation}

\subsection{Failure to Remask Issue}
In Markovian discrete diffusion models with an absorbing kernel defined as $\mathbf{Q}_t=\left(1-\beta_t\right) \mathbf{I}+\beta_t \mathbf{1} e_m^{\top}$, the corresponding marginal kernel is given by $\overline{\mathbf{Q}}_t=\left(1-\alpha_t^*\right) \mathbf{I}+\alpha_t^* \mathbf{1} e_m^{\top}$, where $\alpha_t^*=1-\prod_{s=1}^t\left(1-\beta_s\right)$, Using the closed-form posterior in Equation~\ref{eq:markov_posterior}, we obtain:

\begin{equation}
q\left(\mathbf{x}_{t-1} \mid \mathbf{x}_t, \mathbf{x}_0=\mu_\theta(\mathbf{x}_t)\right)= \begin{cases}\operatorname{Cat}\left(\mathbf{x}_{t-1} ; \mathbf{x}_t\right) & \mathbf{x}_t \neq e_m \\
\operatorname{Cat}\left(\mathbf{x}_{t-1} ; \frac{\alpha_{t-1}^* e_m+\left(\alpha_{t}^*-\alpha_{t-1}^*\right) \mu_\theta(\mathbf{x}_t)}{\alpha_{t}^*}\right) & \mathbf{x}_t=e_m\end{cases}
\end{equation}

This expression highlights an issue: once a token is unmasked (i.e., $\mathbf{x}_t \neq e_m$ ), it is deterministically copied to $\mathbf{x}_{t-1}$, regardless of the model prediction $\mu_\theta\left(\mathbf{x}_t\right)$. As a result, the model is unable to revise early mistakes--an issue we refer to as \textbf{failure to remask}.

Non-Markovian discrete diffusion formulation defines the posterior as $q\left(\mathbf{x}_{t-1} \mid \mu_\theta\left(\mathbf{x}_{t: T}, t\right)\right)$, removing the Markovian constraint and allowing more flexible transitions. By appropriately integrating inductiva bias into the design of \(\mu_\theta\)(e.g. latent truncation as described in Section~\ref{sec:latent_trunc}), the model is able to revisit and potentially forget earlier errors during generation.
\section{Related Work}

\paragraph{Discrete Diffusion and Discrete Flow Matching.}
Diffusion models \citep{ho2020denoisingdiffusionprobabilisticmodels} generate data by learning a reverse (denoising) process to invert a fixed forward (noising) Markov chain. \citet{d3pm} first extended such models to discrete data (D3PM) by defining uniform and absorbing diffusion kernels on finite state spaces. Subsequent work introduced improved parameterizations, such as data distribution ratio estimation \citep{lou2024discretediffusionmodelingestimating}, drawing parallels with score matching \citep{song2021scorebasedgenerativemodelingstochastic}. 
Despite their efficacy, these methods typically rely on a Markov chain, focusing on denoising from a single noisy state \(\mathbf{x}_t\). By contrast, our approach \textbf{breaks} the Markovian assumption and conditions on the entire future trajectory \(\mathbf{x}_{t:T}\), providing more robust denoising and broader generative capabilities.

Flow matching \citep{lipman2023flowmatchinggenerativemodeling, tong2024improvinggeneralizingflowbasedgenerative} learns a continuous transformation from noise to data via an ODE governed by a vector field. Recent extensions handle discrete data \citep{gat2024discreteflowmatching, davis2024fisherflowmatchinggenerative, stark2024dirichletflowmatchingapplications}. While these methods circumvent explicit Markovian noising, they often require continuous flow formulations and specialized training objectives. In contrast, our {non-Markovian discrete diffusion} remains within the discrete diffusion paradigm, retains a straightforward variational objective, and integrates naturally with causal modeling.

\paragraph{Autoregressive Models.}
Autoregressive Transformers \citep{vaswani2017attention, chowdhery2023palm, touvron2023llama} have become foundational in language modeling, producing tokens sequentially conditioned on preceding context. While highly effective for unidirectional, left-to-right generation, they often struggle with tasks requiring intermediate modification or bidirectional reasoning. Within our framework, CaDDi-AR represents a specialized variant that integrates causal (autoregressive) decoding with diffusion-based iterative denoising. This hybrid design enables CaDDi-AR to combine the strengths of both paradigms—efficient left-to-right token generation and flexible multi-step refinement.

\paragraph{Integrating Autoregression with Diffusion and Flow Matching.}
Several works~\citep{DART, he2024calmflowvolterraflowmatching} try to combine diffusion or flow matching with causal transformers for improved generation. Specifically, DART \citep{DART} employs a non-Markovian trajectory to let a transformer model entire sequences of diffusion states. Our approach {further refines} this idea in two ways: (\textit{i}) we focus on discrete non-Markovian diffusion with explicit multi-step conditioning, and (\textit{ii}) we provide a direct path for adapting \textit{pretrained} LLMs, thus combining the strengths of large-scale language model pretraining with the controllability of discrete diffusion.

\paragraph{Non-Markovian Reverse Process in Physical System.}
Using a Non-Markov reverse process to recover the distribution introduced by Markovian forward process is not a new idea. In physics, many systems exhibit this property. \textit{Langevin Dynamics:} Although the forward motion of a Brownian particle (with velocity and position) can be Markovian in the full state space, attempts to reverse the position-only dynamics often require the history of the system to account for friction or random kicks \citep{gardinerstochastic, van1992stochastic}. \textit{Quantum Processes:} Tracing out environmental degrees of freedom can yield a Markovian forward evolution, but reconstructing the entire global state upon reversal introduces non-Markovian memory effects.

\paragraph{Non-Markovian Discrete Diffusion.}
Relaxing the Markovian assumption enables a more flexible and expressive diffusion framework. Prior work, DNDM~\citep{chen2024fast}, focuses on accelerating inference by introducing a predetermined transition time set, enabling a training-free sampling algorithm that significantly reduces the number of function evaluations. This efficiency gain is achieved by breaking the Markovian constraint. Concurrently, ReMDM~\citep{wang2025remasking} proposes a non-Markovian discrete diffusion process that allows remasking of previously generated tokens during inference, enabling iterative refinement. In contrast, CaDDi introduces a more general non-Markovian discrete diffusion framework that explicitly models the entire generation trajectory. 



\setcounter{section}{2}
\section{Model Implementation Details}
\label{sec:implement}
\subsection{Context Window and Latent Compression}
\label{sec:latent_compression}
A practical limitation when using causal language models within our non-Markovian discrete diffusion framework arises from their inherently bounded context windows. Traditional causal transformers can process sequential inputs effectively only up to a fixed length, which restricts the number of latent timesteps that can be accommodated-denoted here as $m$. However, non-Markovian diffusion processes often require conditioning on extensive historical trajectories, frequently exceeding this limit.

To address this constraint, we introduce a general latent compression operator, $\Gamma\left(\mathbf{x}_{1: T}\right)$, which maps the full latent sequence $\mathbf{x}_{1: T}$ to a compressed form that fits within the model's context window. Accordingly, the reverse denoising function is expressed as:
\begin{equation}
\boldsymbol{\mu}_\theta\left(\mathbf{x}_{t:T}, t\right) := \boldsymbol{\mu}_\theta\left(\Gamma\left(\mathbf{x}_{t:T}\right), t\right)
\end{equation}

\subsubsection{Latent Truncation}
\label{sec:latent_trunc}
The simplest yet effective instantiation of $\Gamma$ is latent truncation, defined as
\begin{equation}
\Gamma_{\text {trunc }}\left(\mathbf{x}_{t: T}\right):=\operatorname{FlattenTime}\left(\mathbf{x}_t, \mathbf{x}_{t+1}, \ldots, \mathbf{x}_{t+m-1}\right),
\end{equation}
where \(\operatorname{FlattenTime}\) denotes the operation of temporally flattening the sequence into a one-dimensional input trajectory suitable for the causal language model.
For a model with a context window limit of \(m\), this strategy retains only a fixed-length segment from the most recent \(m\) timesteps, discarding earlier latents.

Such selective truncation inherently emphasizes the most recent—and typically more informative—states, which are often less affected by accumulated inference noise. Notably, by intentionally discarding earlier latents, the model is allowed to “forget” potentially noisy or erroneous history, an idea related to selective re-masking approaches in diffusion modeling~\citep{wang2025remasking}. This forgetfulness acts as an implicit regularization mechanism, helping the model focus on more stable and relevant information while reducing the risk of propagating stale or corrupted context during inference.

\subsubsection{Trajectory Re-composition}
\label{sec:latent_compose}
While latent truncation is efficient, it can discard valuable long-range information beyond timestep $t+m$. To mitigate this, we introduce trajectory re-composition, a complementary compression strategy that integrates information across the full latent sequence before applying truncation. This method first aggregates latent information through a sequential integration operation and then retains only the most recent $m$ composite states:
\begin{equation}
\begin{aligned}
    \mathbf{x}_t^{\curvearrowright} &:= \mathbf{x}_t \oplus \mathbf{x}_{t+1} \oplus \cdots \oplus \mathbf{x}_T\\
    \Gamma_{\mathrm{rec}}\left(\mathbf{x}_{1: T}\right)&:=\operatorname{FlattenTime}\left(\mathbf{x}_t^{\curvearrowright}, \mathbf{x}_{t+1}^{\curvearrowright}, \ldots, \mathbf{x}_{t+m-1}^{\curvearrowright}\right)
\end{aligned}
\end{equation}
Here, $\oplus$ denotes an element-wise integration operator. For masked-like diffusion models, this operator replaces each element with the most recently unmasked token when applicable. This re-composition process enables earlier timesteps to incorporate contextual signals from future states, effectively compressing the trajectory into a form suitable for the limited context window without losing essential long-range dependencies.

\subsection{Causal Bidirectional
Augmentation}
\label{sec:causal_bidrectional}

\begin{figure}[t]
    \centering
    \begin{subfigure}[c]{0.3\textwidth}
    \centering
        \includegraphics[width=\textwidth]{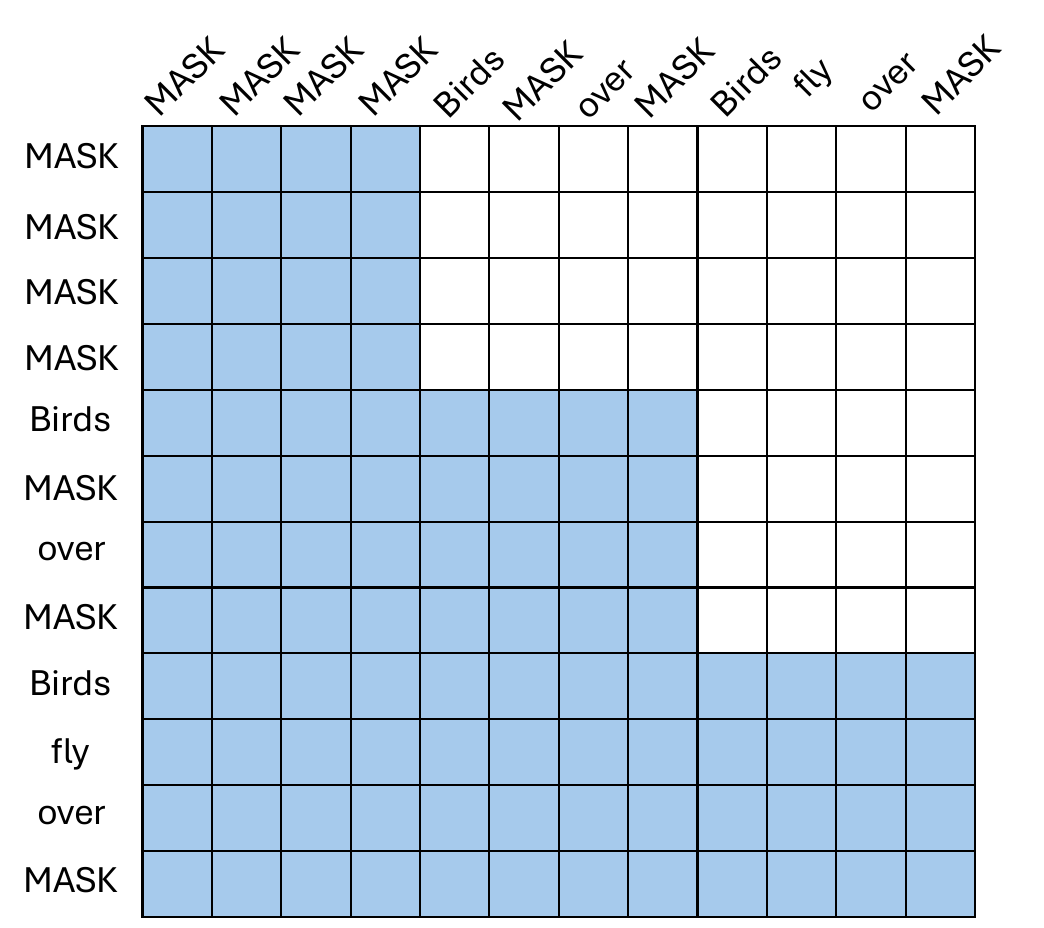}
        \caption{Block-level causal attention mask}
        \label{fig:caddi_block_mask}
    \end{subfigure}
    \hfill
    \begin{subfigure}[c]{0.65\textwidth}
    \centering
        \includegraphics[width=\textwidth]{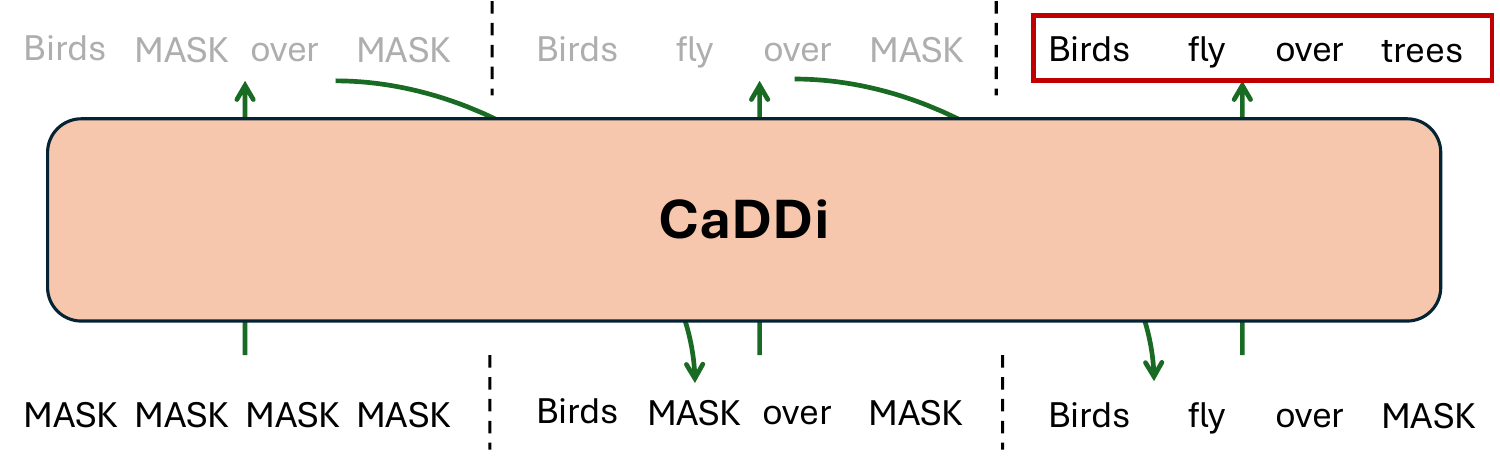}
        \caption{CaDDi generation with block-level causal mask}
        \label{fig:caddi_block_generation}
    \end{subfigure}
    \caption{Illustration of vanilla block-wise generation of CaDDi. Figure~\ref{fig:caddi_block_mask} shows the attention mask of vanilla block-wise generation. The block-level causal mask allows bidirectional attention within each time point and causal attention over the time points. Figure ~\ref{fig:caddi_block_generation} shows the generation scheme. Note that the model itself predicts the clean data $x_0$ in practice but the figure highlights the next sampled time point (in color) for clarity. }
    \label{fig:caddi_block}
\end{figure}
Conventional discrete diffusion models often employ transformers with bidirectional attention across the sequence dimension. In the context of non-Markovian discrete diffusion, a natural extension of this framework is to incorporate a block-wise causal attention mask, as illustrated in Figure~\ref{fig:caddi_block}. This block-level causal mask enables bidirectional attention within a single timepoint and enforces causal dependencies across different timepoints. 

However, unlike token-level causal attention (CaDDi-AR), block-wise causal attention of CaDDi is not readily compatible with standard pretrained causal language models, which adopt token-level causal mask and next-token prediction scheme. Moreover, as it involves an irregular attention pattern, it cannot take advantage of existing infrastructure—such as efficient implementations like FlashAttention \cite{dao2022flashattention}.

For CaDDi, our goal is to enable one-shot block-level generation while preserving bidirectional modeling capabilities. Directly applying a token-level causal mask to predict the next block in one shot would eliminate bidirectional modeling, as the final timepoint’s tokens in the context window would lack visibility into subsequent tokens (see Figure~\ref{fig:bidirectional_attn}). So here we propose an alternative - \textit{causal bidirectional augmentation}, which involves repeating the final timepoint in the context window. As shown in Figure~\ref{fig:caddi_bidir_generation}, this approach preserves compatibility with the standard token-level causal mask and supports block-wise generation with bidirectional modeling. Our method also bears resemblance to prior work~\citep{springer2024repetitionimproveslanguagemodel} on repetition-based improvements in language modeling.

\begin{figure}[t]
    \centering
        \includegraphics[width=0.6\textwidth]{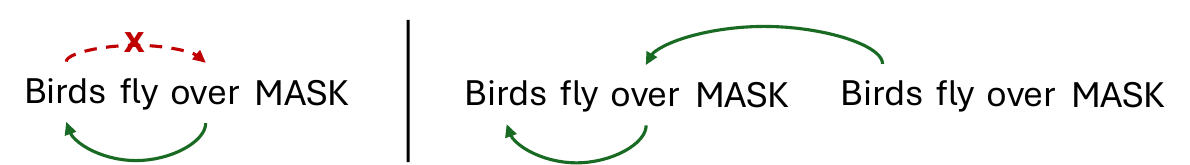}
    \caption{Illustration of causal bidirectional
augmentation. On the left, direct token-level causal attention is applied to the original sequence, where each token attends only to preceding tokens. On the right, the sequence is repeated, allowing tokens in the second copy to attend to their counterparts in the first, thereby approximating bidirectional attention with token-level causal attention.}
    \label{fig:bidirectional_attn}
\end{figure}

\begin{figure}[t]
    \centering
    \begin{subfigure}[c]{0.3\textwidth}
    \centering
        \includegraphics[width=\textwidth]{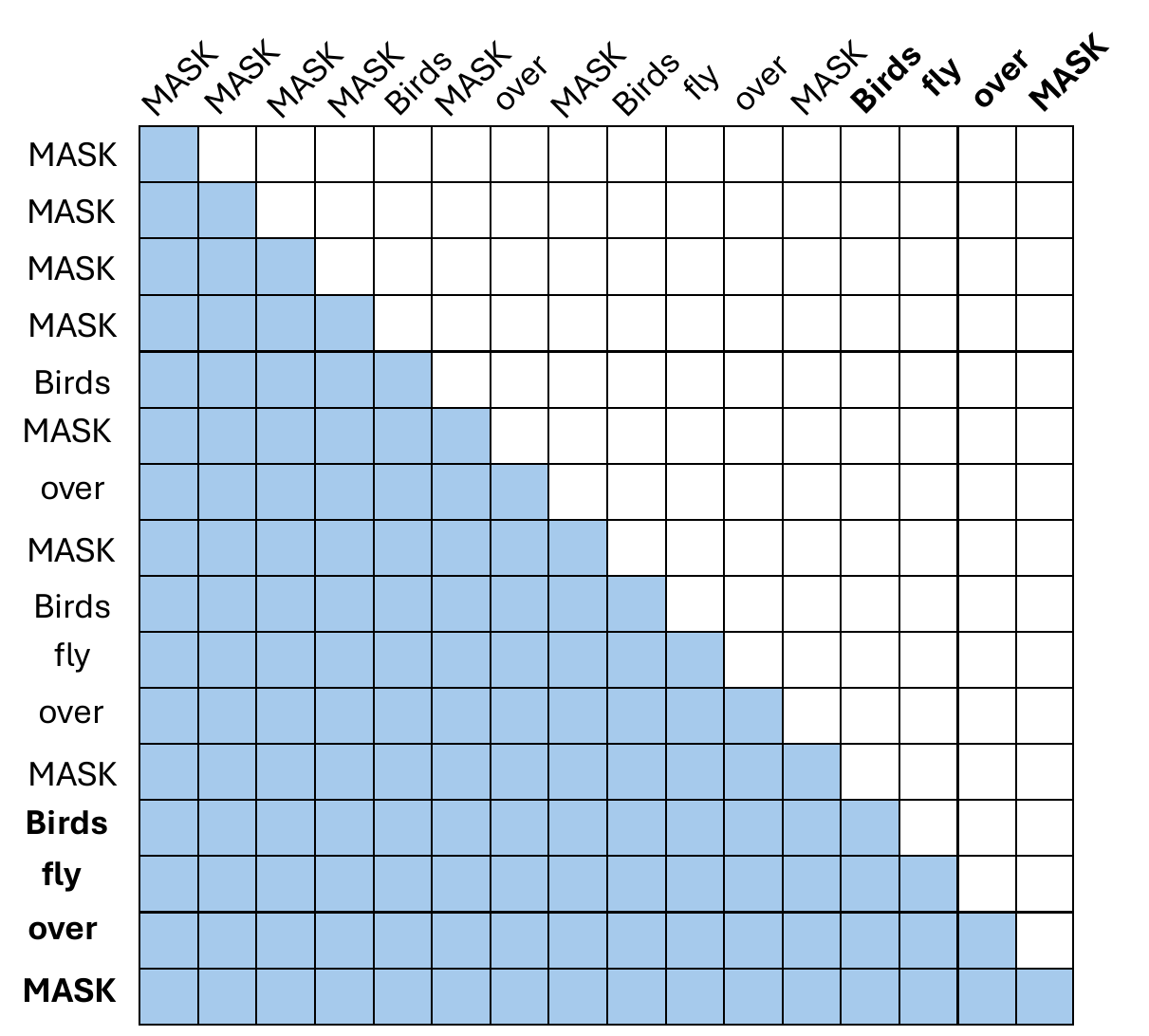}
        \caption{Token Level causal attention mask}
        \label{fig:caddi_bidir_mask}
    \end{subfigure}
    \hfill
    \begin{subfigure}[c]{0.65\textwidth}
    \centering
        \includegraphics[width=\textwidth]{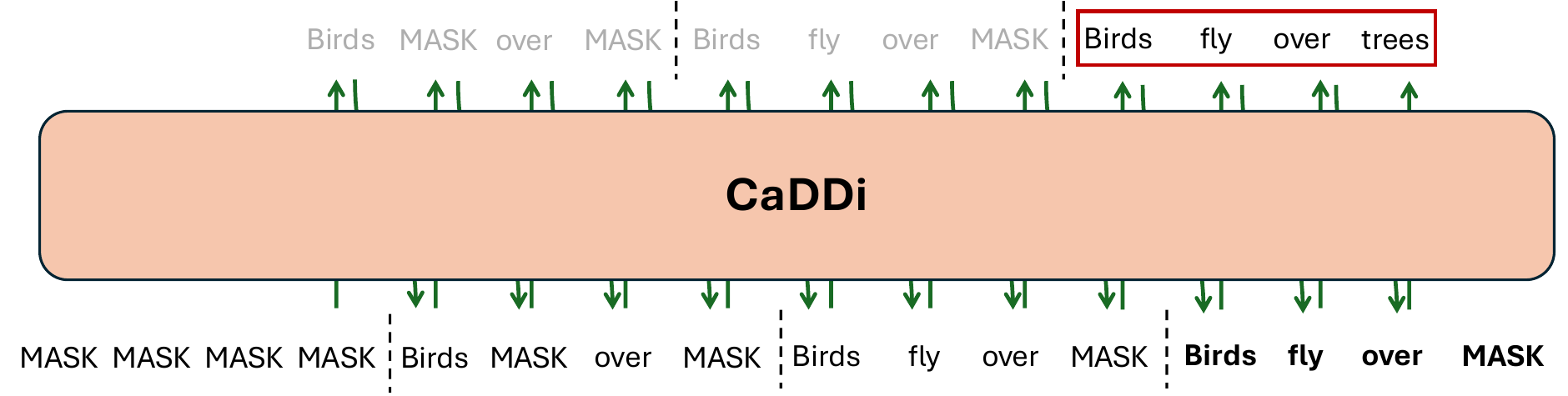}
        \caption{CaDDi Generation with causal bidirectional augmentation}
        \label{fig:caddi_bidir_generation}
    \end{subfigure}
    \caption{Illustration of causal bidirectional
augmentation. The last time point is repeated and highlighted in \textbf{bold} to approximate bidirectional attention within a causal framework. Figure~\ref{fig:caddi_bidir_mask} shows the token-level causal attention mask, identical to the standard mask used in causal language models. Figure~\ref{fig:caddi_bidir_generation} illustrates the generation process: for the token at position \(i\), the model attends to all tokens in the generative trajectory and predicts the token at position \((i+1)\) in \(\mathbf{x}_0\), consistent with the autoregressive nature of causal language models. Both the attention mask and generation process are fully compatible with causal architectures, making CaDDi a natural extension. Note that, while the model predicts the clean data $x_0$ in practice; the next time point is shown gray here for illustrative simplicity.}
    \label{fig:caddi_bidir}
\end{figure}

\subsection{CaDDi-AR Token-Level Factorization}
As discussed in Section~\ref{sec:caddi_ar}, CaDDi-AR models token-level dependencies within each denoising step by adopting an autoregressive factorization over the token dimension. This approach builds on prior work such as DART~\citep{DART}, which proposed token-wise autoregression for improving expressiveness in discrete diffusion.

Formally, instead of modeling the full sequence distribution in a single step, we decompose the reverse process as:

\begin{equation}
p_\theta\left(\mathbf{x}_{t-1} \mid \mathbf{x}_{t: T}\right)=\prod_{i=1}^L p_\theta\left(\mathbf{x}_{t-1}^i \mid \mathbf{x}_{t-1}^{<i}, \mathbf{x}_{t: T}\right)
\end{equation}

where $\mathbf{x}_{t-1}^{<i}=\left(\mathbf{x}_{t-1}^1, \ldots, \mathbf{x}_{t-1}^{i-1}\right)$. This sequential factorization enables more accurate modeling of intrastep token dependencies compared to fully factorized approaches.

In our formulation, we further adopt the $x_0$-parameterization, predicting the clean sequence autoregressively and applying the forward corruption kernel to reconstruct $\mathbf{x}_{t-1}$. This leads to the following form:

\begin{equation}
p_\theta\left(\mathbf{x}_0 \mid \mathbf{x}_{t: T}\right)=\prod_{i=1}^L p_\theta\left(\mathbf{x}_0^i \mid \mathbf{x}_0^{<i}, \mathbf{x}_{t: T}\right)
\end{equation}

where $\mathbf{x}_0$ denotes the clean target sequence. This structure aligns well with standard causal language models, allowing autoregressive decoding over tokens while conditioning on the latent trajectory $\mathbf{x}_{t: T}$ as a static prompt.

In implementation, CaDDi-AR is trained using standard left-to-right causal masking, with the latent variables provided as additional context. During inference, the model samples $\mathbf{x}_0$ autoregressively at each timestep and applies the corruption kernel to obtain $\mathbf{x}_{t-1}$. To reduce computational cost, we employ semi-speculative decoding (Section~\ref{sec:semi-sepculative}) by partially reusing predictions from previous steps and verifying them in parallel.

This token-level decoding strategy provides finer granularity for generation, and is particularly effective in tasks requiring fluent, coherent text output or strong local consistency.




\setcounter{section}{4}
\section{Experiment Details}
\subsection{LM1B Dataset Experiment Details}\label{sec:appendix-metrics}

\paragraph{Dataset Preprocessing.}
We follow the preprocessing setup introduced in DiffusionBERT~\citep{he2022diffusionbertimprovinggenerativemasked}, using the One Billion Word Benchmark~\citep{chelba2014billionwordbenchmarkmeasuring}. Sentences are tokenized using the \texttt{bert-base-uncased} tokenizer with a vocabulary size of 30,522. All sequences are padded or truncated to a fixed length of 128 tokens during training.

\paragraph{Model Configuration.}
All models are based on a 12-layer Transformer decoder architecture with a hidden size of 768 and 12 attention heads. For D3PM~\citep{d3pm}, MDLM~\citep{mdlm}, and SEDD~\citep{lou2024discretediffusionmodelingestimating}, we adopt an absorbing diffusion kernel with a log-linear noise schedule. For CaDDi and CaDDi-AR, we use the absorbing-state forward kernel described in Section~\ref{prop:equivalence}, with total diffusion steps set to \(T = 64\). Note that models such as MDLM and SEDD are trained in continuous time, whereas our models operate in discrete time. CaDDi uses a context window of 5 and applies latent truncation as described in Section~\ref{sec:latent_trunc}. In this experiment, CaDDi-AR is trained entirely from scratch without leveraging any pretrained language model weights.

\paragraph{Training Details.}
Models are trained using AdamW with a learning rate of 3e-4, 2500 warm-up steps. All models use a batch size of 512 and train for 1000K steps. Our models are trained on 4 NVIDIA H100 GPUs with mixed precision.

\paragraph{Inference and Sampling.}
We evaluate generative perplexity using pretrained oracle models (GPT-2, LLaMA-2-7B, and LLaMA-3-3B) under three sampling temperatures: \(T = 1.0\), \(T = 0.7\), and \(T = 0.5\). For each evaluation, sequences are generated according to the length distribution of the dataset, and average perplexity is computed under the oracle model. All diffusion-based models use 64 denoising steps during inference. For CaDDi-AR, we additionally explore semi-speculative decoding to reduce sampling latency without sacrificing quality.

\paragraph{Evaluation Metrics.}
We report \textbf{oracle-based generative perplexity (Gen PPL)} as our primary metric for evaluating generation quality. Gen PPL is computed using a separate, pretrained causal language model, which assesses how well it can predict the next token in the generated sequence. Intuitively, lower perplexity indicates that the oracle model finds the generated text more coherent and predictable given the context—implying better fluency and consistency.

To mitigate known issues with perplexity, such as its tendency to reward repetitive or degenerate outputs~\citep{holtzman2020curiouscaseneuraltext}, we adopt \textbf{guided generative perplexity}. Specifically, we prepend a natural language prompt---\texttt{Does the following sentence make sense: }---to each sequence before evaluation. This encourages the oracle model to assess coherence in a more human-aligned fashion, reducing the risk of falsely low perplexity on poor-quality outputs.

In addition to quality, we assess diversity using \textbf{token-level entropy} over the generated output distribution without applying temperature scaling. This captures how varied the model’s outputs are across generations.

\subsection{Amazon Polarity Conditional Generation Details}\label{sec:amazon_process}

\paragraph{Dataset and Preprocessing.}
We use the Amazon Polarity dataset~\citep{mcauley2013hidden}, a large-scale binary sentiment classification corpus consisting of approximately 3.6 million product reviews labeled as either \texttt{positive} or \texttt{negative}. All reviews are tokenized using the \texttt{bert-base-uncased} tokenizer and truncated to a maximum length of 128 tokens. To enable conditional generation, we prepend a natural language sentiment prompt to each review. This prompt takes the form of a simple prefix indicating the intended sentiment (e.g., \texttt{positive} or \texttt{negative}), allowing standard causal language models—such as GPT-2—to generate sentiment-aligned outputs. Example formatted inputs are shown below:

\begin{itemize}
    \item \textbf{This is a positive review:}\\
    \textbf{Title:} \textit{Great!!}\\
    \textbf{Content:} \textit{"This product is amazing! The quality exceeded my expectations and I will definitely buy again."}
    
    \item \textbf{This is a negative review:}\\
    \textbf{Title:} \textit{Very disappointing}\\
    \textbf{Content:} \textit{"It broke after a week and customer service was not helpful."}
\end{itemize}

\paragraph{Conditional Generation Setup.}
We adapt our model to perform sentiment-controlled generation by training a unified denoising network \(\mu_\theta(\mathbf{c}, \mathbf{x}_{t:T})\), where \(\mathbf{c}\) denotes a conditioning input (e.g., a sentiment label prompt) and \(\mathbf{x}_{t:T}\) is the observed noisy trajectory. During training, we simulate both conditional and unconditional modes by randomly masking the conditioning input \(\mathbf{c}\). This enables the model to learn both \(\mu_\theta(\mathbf{c}, \mathbf{x}_{t:T})\) and \(\mu_\theta(\mathbf{x}_{t:T})\) within a single parameterization.

At inference time, we generate sentiment-aligned text using either direct conditioning (i.e., sampling from \(q(\mathbf{x}_{t-1} \mid \mu_\theta(\mathbf{c}, \mathbf{x}_{t:T}))\)) or classifier-free guidance. In the latter case, we apply the reweighted sampling distribution \(\tilde{q}(\mathbf{x}_{t-1} \mid \mu_\theta(\mathbf{c}, \mathbf{x}_{t:T}))\) as defined in Equation~\ref{eq:cfg}, which balances conditional and unconditional predictions using a guidance scale \(\gamma\). This allows finer control over sentiment alignment during generation.

\paragraph{Classifier-Free Guidance.}
We extend classifier-free guidance (CFG) to the unsupervised discrete diffusion setting by reweighting the sampling distribution at each reverse step. Specifically, we define a guided transition distribution over $\mathbf{x}_{t-1}$ as:

\begin{equation}
\label{eq:cfg}
\tilde{q}\left(\mathbf{x}_{t-1} \mid \mu_\theta(\mathbf{c}, \mathbf{x}_{t:T})\right) \propto 
\frac{q\left(\mathbf{x}_{t-1} \mid \mu_\theta(\mathbf{c}, \mathbf{x}_{t:T})\right)^{\gamma}}
{q\left(\mathbf{x}_{t-1} \mid \mu_\theta(\mathbf{x}_{t:T})\right)^{\gamma - 1}},
\end{equation}

where $\mathbf{c}$ denotes a conditioning signal (e.g., a sentiment label or prompt), and $\mu_\theta(\cdot)$ is the denoising network predicting the clean input $\mathbf{x}_0$ from the noisy trajectory $\mathbf{x}_{t:T}$. The numerator corresponds to the conditional denoising distribution, while the denominator represents the unconditional variant, where the conditioning input $\mathbf{c}$ is masked out. This formulation smoothly interpolates between conditional and unconditional behavior, analogous to CFG in continuous diffusion.

The guidance scale $\gamma \geq 1$ controls the strength of conditioning: when $\gamma = 1$, the model performs standard conditional generation; as $\gamma$ increases, the model places more emphasis on aligning with the conditioning signal, potentially improving controllability at the cost of diversity or fluency. In practice, we find that moderate values such as $\gamma = 1.0$ or $1.25$ strike a good balance between alignment and generation quality.

To enable this guidance, we train the denoising model jointly on both conditional and unconditional modes by randomly masking the conditioning input $\mathbf{c}$ during training. This allows the model to learn both behaviors within a single parameterization.

\paragraph{Evaluation.}
To evaluate sentiment alignment, we use a publicly available DistilBERT classifier fine-tuned on the Amazon Polarity dataset\footnote{\url{https://huggingface.co/kaustavbhattacharjee/finetuning-DistillBERT-amazon-polarity}}. For each sentiment label, we generate 1,000 samples using top-\(k\) sampling with \(k = 50\) and temperature \(T = 1.0\), across 64 denoising steps. Sentiment accuracy is computed as the percentage of generated samples whose predicted label matches the intended conditioning prompt. Results are reported in Table~\ref{tab:sentiment-comparison}.

For the GPT-2 baseline, we condition generation on the same prepended prompt used in our model and apply the same top-\(k\) sampling and temperature settings for consistency.

\subsection{Text8 Dataset Experiment Details}

\paragraph{Dataset Preprocessing.}
We use the standard Text8 dataset, a 100M character-level corpus derived from Wikipedia. Following prior work~\citep{d3pm, shi2024simplified}, we split the raw text into non-overlapping sequences of 256 characters. No tokenization is applied—each character is treated as a discrete token from an alphabet of size 27 (26 letters + space). We use the first 90\% of the dataset for training and the remaining 10\% for validation and testing.

\paragraph{Model Setup.}
All models use a 12-layer Transformer with hidden size 768 and 12 attention heads. We adopt the absorbing-state forward kernel described in Section~\ref{prop:equivalence}. For CaDDi, we use $T = 64$ denoising steps. The default context window is set to 5, and we apply latent truncation and trajectory recomposition as described in Section~\ref{sec:latent_compression}. Baselines (D3PM, SEDD, MDLM, UDLM) are either reimplemented or taken from their official codebases using the configuration for fair comparison.

\paragraph{Training Details.}
Models are trained using AdamW with a learning rate of 3e-4 and 2,500 warm-up steps. We use a batch size of 512 and train for 1M steps. We use the simplified ELBO objective for absorbing kernels as described in Equation~(8), which reduces to a weighted cross-entropy loss over clean targets.

\paragraph{Evaluation Metrics.}
We report bits-per-character (BPC) on the test set, computed as:
\[
\text{BPC} = -\frac{1}{L} \sum_{i=1}^L \log_2 p(x_i),
\]
where $L$ is the sequence length and $p(x_i)$ is the predicted probability of the $i$-th character. For diffusion-based models, we compute a variational upper bound on the log-likelihood using the ELBO objective. For autoregressive baselines, we report the true NLL.

\textbf{Note}: We do not report likelihood-reltaed metrics for CaDDi-AR due to its token-level autoregressive decomposition under \(\mathbf{x}_0\)-parameterization \(p_\theta\left(\mathbf{x}_{0} \mid \mathbf{x}_{t: T}\right)=\prod_{i=0}^L p_\theta\left(\mathbf{x}_{0}^i \mid \mathbf{x}_{0}^{<i}, \mathbf{x}_{t: T}\right)\). This formulation makes direct evaluation of \(\log p_\theta\left(\mathbf{x}_{t-1}^i \mid \mathbf{x}_{t: T}\right)\) intractable, as it requires marginalizing over all possible prefix sequences $\mathbf{x}_0^{<i}$. A tractable lower bound can be estimated via:

\begin{equation}
\log p_\theta\left(\mathbf{x}_{t-1}^i \mid \mathbf{x}_{t: T}\right) \geqslant \E_{\mathbf{x}_0^{<i} \sim p_\theta\left(\mathbf{x}_0^{<i} \mid \mathbf{x}_{t: T}\right)} \log p_\theta\left(\mathbf{x}_{t-1}^i \mid \mathbf{x}_0^{<i}, \mathbf{x}_{t: T}\right)
\end{equation}

but this introduces an additional approximation gap in likelihood estimation, making reported values less directly comparable.

\paragraph{Comparison Notes.}
Likelihood estimation of diffusion model is sensitive to the discretization of timesteps: as the number of steps increases, perplexity typically decreases. To ensure a fair comparison, all diffusion-based models are evaluated using 64 denoising steps for consistency. Some prior works report results under continuous-time settings or with 1000-step discretizations; we include these numbers for reference but highlight the corresponding rows in gray in Table~\ref{tab:text8} to indicate that they are comparable.

\subsection{General Language Reasoning Dataset Details}

\paragraph{Dataset Overview.}
We evaluate CaDDi-AR on a set of natural language understanding benchmarks covering commonsense reasoning, factual QA, and reading comprehension:
\begin{itemize}
    \item \textbf{ARC-Challenge / ARC-Easy}~\citep{arc}: multiple-choice science questions.
    \item \textbf{BoolQ}~\citep{clark2019boolq}: binary reasoning dataset based on a short context.
    \item \textbf{PIQA}~\citep{Bisk2020}: commonsense physical reasoning questions with two-choice answers.
    \item \textbf{RACE}~\citep{lai-etal-2017-race}: multi-choice reading comprehension from English exams.
    \item \textbf{Social IQA}~\citep{sap2019socialiqacommonsensereasoningsocial}: social commonsense reasoning dataset.
    \item \textbf{LAMBADA}~\citep{lambada_dataset}: cloze-style word prediction requiring broad context understanding.
\end{itemize}


\paragraph{Fine-tuning Setup.}
Following~\citep{nie2025scalingmaskeddiffusionmodels},we fine-tune CaDDi-AR on ShareGPT\footnote{https://sharegpt.com/} dataset from a pretrained QWen-1.5B checkpoint using our diffusion-based objective with \(T = 64\) steps. The model is trained using AdamW with a learning rate of 5e-5, a batch size of 64 with gradient accumulation, and 20K total steps. We adopt the absorbing kernel formulation with simplified ELBO as the training loss.

\paragraph{Inference Procedure.}
We evaluate CaDDi-AR on standard natural language reasoning tasks using the Language Model Evaluation Harness~\citep{eval-harness}, a widely adopted framework for assessing pretrained language models on benchmark datasets. Following common practice, we convert each task into a text completion format and measure the log-likelihood of the correct answer under the model. The final prediction is selected as the choice with the highest likelihood. This approach ensures consistency across models with different architectures (e.g., ARMs and MDMs). We report \textbf{accuracy} as the primary evaluation metric for all datasets.

\paragraph{Baselines.}
We compare CaDDi-AR against several established language models of comparable scale. Specifically, we include GPT-2 (1.5B parameters), TinyLLaMA (1.1B), and MDM (1.1B), a recently proposed diffusion-based language model. To ensure a fair comparison, all baselines are evaluated using the same prompt formatting and inference procedure as described above. For MDM, we directly use the performance numbers reported in the original paper.

\setcounter{section}{6}
\section{Additional Experiments and Ablation Study}
\label{sec:addition_experiment}




\subsection{Effect of Sampling Steps on Generative Quality}
\label{sec:scaling_lm1b}
\begin{figure}[H]
    \centering
    \includegraphics[width=\linewidth]{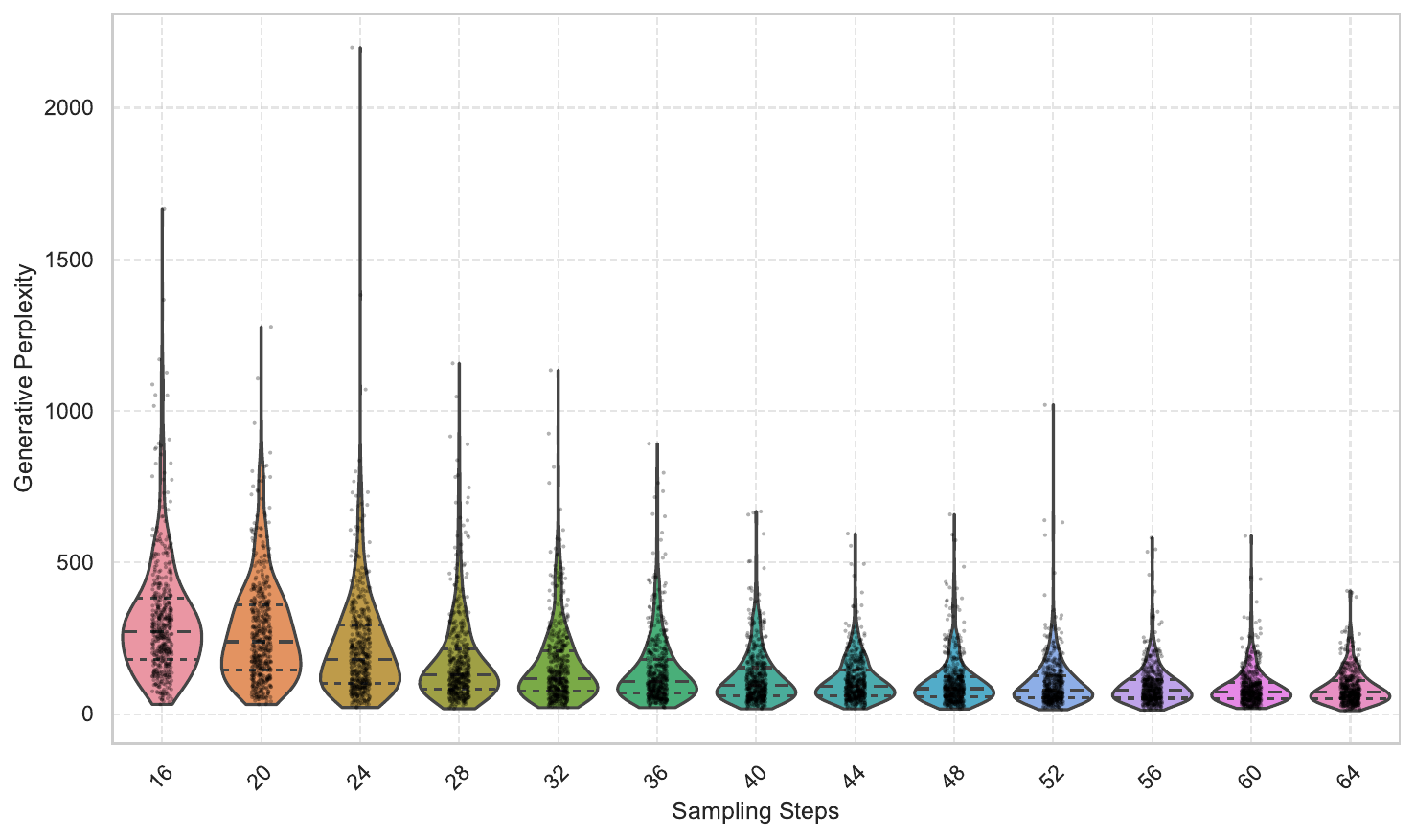}
    \caption{Generative Perplexity Distribution Across Sampling Steps}
    \label{fig:violin_sampling_steps}
\end{figure}
To investigate the relationship between sampling budget and generative quality, we evaluate the generation of CaDDi model trained with LM1B dataset under varying numbers of sampling steps. We report generative perplexity as the evaluation metric, computed from oracle model's likelihood on generated sentences.

Unlike prior work in continuous-time settings~\citep{mdlm, lou2024discretediffusionmodelingestimating}, CaDDi is trained with a fixed discrete timestep schedule ( $T=64$ ). To enable adaptive sampling with fewer denoising steps at inference, we employ a uniform step-skipping strategy. Specifically, for selected timesteps $t$, we directly use the most recent prediction of $\mathbf{x}_0$ to sample from the corresponding latent distribution $q\left(\mathbf{x}_t \mid \mathbf{x}_0\right)$, without invoking the neural network for prediction. This allows efficient generation under a reduced sampling budget while preserving the learned diffusion trajectory.

Figure~\ref{fig:violin_sampling_steps} presents a violin plot illustrating the distribution of perplexity scores as a function of the number of denoising steps. As the number of sampling steps increases, we observe a consistent reduction in perplexity, indicating improved sample quality. This trend suggests that the model benefits from longer refinement trajectories, with more steps allowing finer-grained correction of uncertainty during generation.

The observed perplexity curve approximately follows a scaling behavior, reminiscent of trends in autoregressive models where performance improves predictably with increased compute or depth. This hints at a potential scaling law in non-Markovian discrete diffusion generation: sample quality improves smoothly with increased inference budget. Future work may explore formalizing this relationship and connecting it to theoretical underpinnings of discrete-time inference refinement.

\subsection{Other Ablation Study}
\begin{figure}[t]
    \centering
    \includegraphics[width=0.9\linewidth]{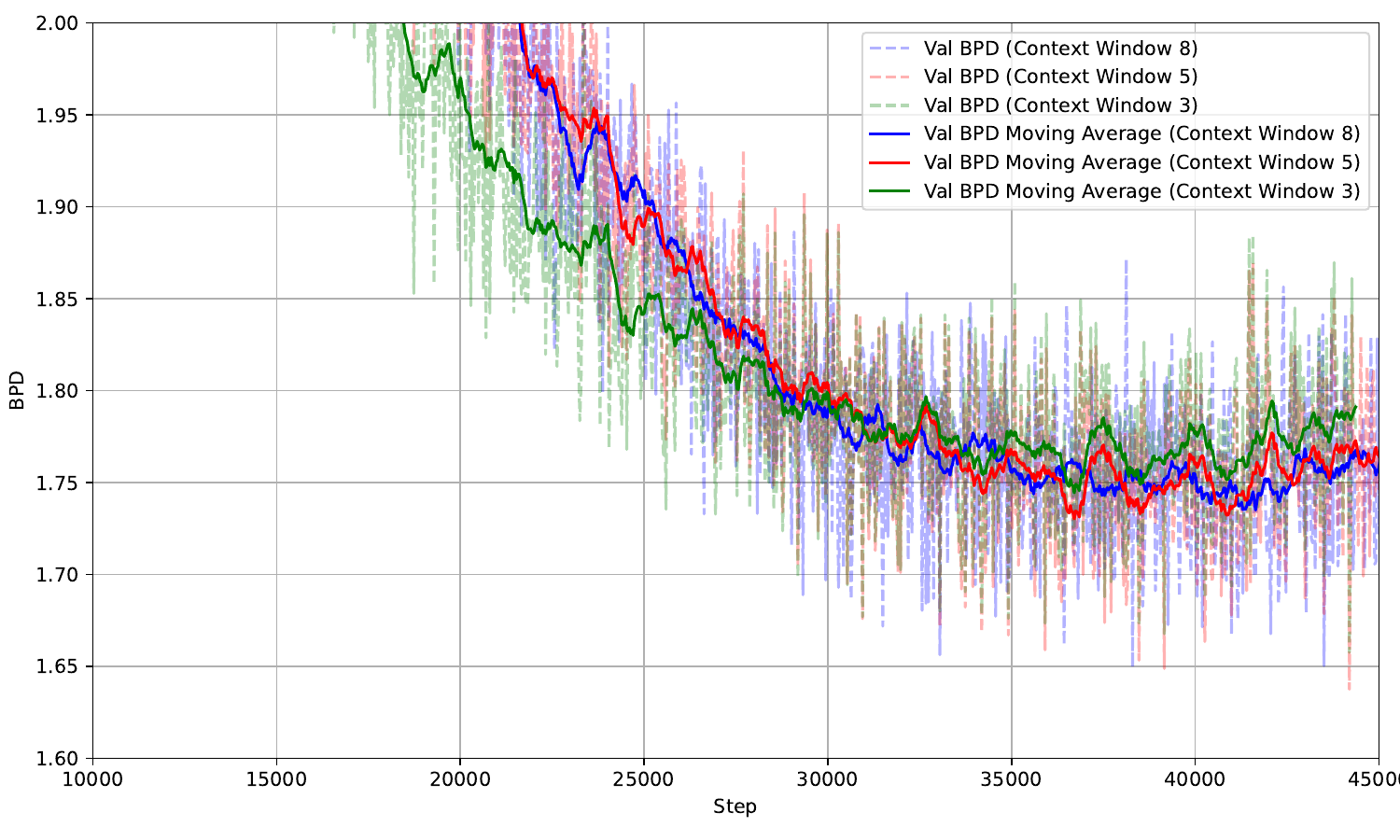}
    \caption{Validation BPD over training steps for different context window sizes (3, 5, and 8). Dashed lines represent raw validation BPD at each step, while solid lines show the smoothed moving average. Larger context windows consistently yield lower BPD, indicating improved modeling of long-range dependencies. However, context window 5 offers a favorable balance between performance and stability.}
    \label{fig:loss_curve}
\end{figure}

To assess the impact of key architectural and training design choices, we conduct an ablation study on a subset of the Text8 dataset, using only the first 10\% of the training data. We focus on three factors: the number of diffusion steps, the choice of positional encoding, and the size of the context window used during latent truncation, as described in Section~\ref{sec:latent_trunc}. Results are shown in Table~\ref{tab:ablation}, using bits-per-dimension (BPD), perplexity, and negative log-likelihood (NLL) as evaluation metrics.

\paragraph{Diffusion steps.} Increasing the number of diffusion steps from 16 to 128 consistently improves performance across all metrics. Specifically, BPD decreases from 2.013 to 1.712, and perplexity drops from 4.04 to 3.28. This suggests that additional refinement steps lead to more accurate modeling of the data distribution, albeit at increased computational cost.

\paragraph{Positional encoding.} We compare our 2D RoPE encoding (used in CaDDi) with 1D RoPE and sinusoidal RoPE variants. The CaDDi encoding achieves the best results, while alternative encodings result in slight performance degradation. This highlights the importance of aligning positional encodings with the inductive biases of the diffusion-based architecture.

\paragraph{Context window.} We vary the context window size during latent truncation and observe a trade-off between context length and model performance. A window size of 8 yields the best results (BPD = 1.740) but incurs higher computational cost. A window size of 5 offers a favorable balance, achieving competitive performance with reduced resource requirements. Accordingly, we adopt it as the default setting in our experiments. We also plot the learning dynamic of different context window size in Figure~\ref{fig:loss_curve}. These findings support the intuition that larger context windows aid in modeling long-range dependencies, though benefits plateau beyond a certain size.

\paragraph{Attention masking.} 
We compare two attention masking strategies: a block-level causal mask and a token-level causal mask combined with Causal Bidirectional Augmentation, as described in Section~\ref{sec:causal_bidrectional}. The two approaches yield nearly identical performance. However, the token-level mask with bidirectional augmentation offers practical advantages: it enables accelerated inference through flash attention and is inherently more compatible with pretrained large language models.

\begin{table}[t]
  \centering
  \small 
  \caption{Ablation study evaluating the impact of diffusion steps, positional encoding, context window size, and attention masking strategies. Default condiguration is denoted with *}
  \label{tab:ablation}
  \begin{tabular}{ll
                S[table-format=1.3]
                S[table-format=1.4]
                S[table-format=1.4]}
    \toprule
    \multicolumn{2}{l}{\textbf{Configuration}} & \textbf{BPD} & \textbf{Perplexity} & \textbf{NLL} \\
    \midrule
    \multirow{4}{*}{\textit{Diffusion steps}} 
      & CaDDi-16 steps    & 2.013 & 4.0362 & 1.3953 \\
      & CaDDi-32 steps    & 1.845 & 3.5925 & 1.2789 \\
      & CaDDi-64 steps*    & 1.751 & 3.3659 & 1.2137 \\
      & CaDDi-128 steps   & 1.712 & 3.2761 & 1.1867 \\
    \midrule
    \multirow{3}{*}{\textit{Positional Encoding}} 
      & 2D RoPE*          & 1.751 & 3.3659 & 1.2137 \\
      & 1D RoPE           & 1.773 & 3.4176 & 1.2290 \\
      & Sinusoidal RoPE   & 1.801 & 3.4846 & 1.2484 \\
    \midrule
    \multirow{3}{*}{\textit{Context Window}} 
      & Window-8          & 1.740 & 3.3404 & 1.2061 \\
      & Window-5*          & 1.751 & 3.3659 & 1.2137 \\
      & Window-3          & 1.791 & 3.4605 & 1.2414 \\
    \midrule
    \multirow{2}{*}{\textit{Attention Mask}} 
      & Block-level causal mask              & 1.747 & 3.3566 & 1.2109 \\
      & \begin{tabular}[c]{@{}l@{}}Token-level causal mask \\ + Causal Bidirectional Aug.\end{tabular}*  & 1.751 & 3.3659 & 1.2137 \\
    \bottomrule
  \end{tabular}
\end{table}

\subsection{Inference Robustness Under Noise Injection}
\begin{figure}[t]
    \centering
    \includegraphics[width=0.7\linewidth]{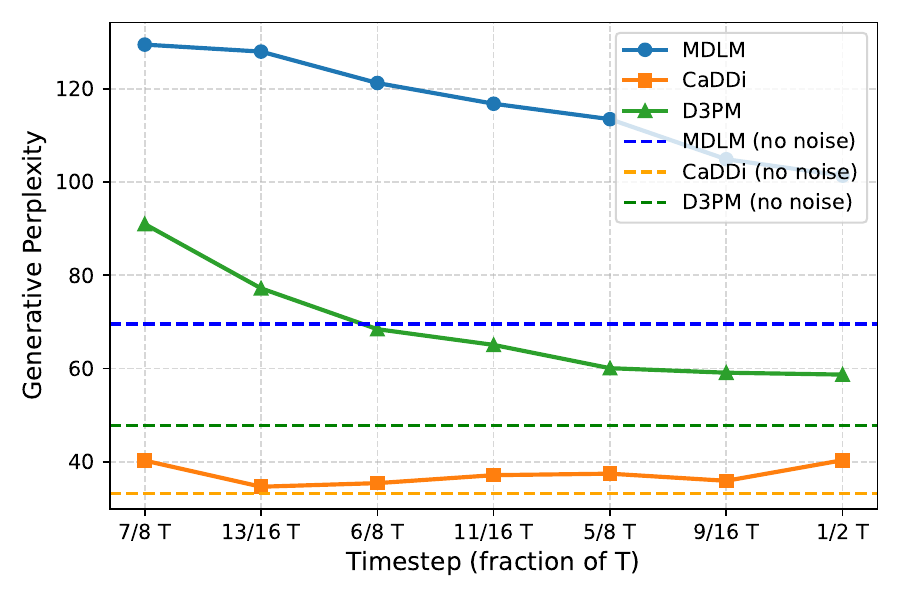}
    \caption{Generation performance under manually injected noise at different timestep}
    \label{fig:robust_in_appendix}
\end{figure}

To further assess the robustness of our proposed CaDDi model during generation, we conduct a controlled perturbation study in which synthetic noise is manually injected into the latent trajectory at various timesteps. This simulates potential inference-time errors and allows us to systematically examine how different models propagate and recover from early-stage inaccuracies.

\textbf{Setup}. At a chosen timestep $t \in\{T / 2,9 T / 16,5 T / 8, \ldots, T\}$, we randomly corrupt a subset of tokens in the latent variable $\mathbf{x}_t$ by replacing them with uniformly specific values from the vocabulary, thereby introducing controlled deviation from the expected latent distribution. We then allow the model to proceed with the remaining reverse steps without additional interventions and evaluate the final output quality.

\textbf{Metrics}. We report generative perplexity computed against oracle LMs (e.g., GPT-2), following the evaluation procedure in Section 5.2. Lower perplexity indicates stronger resilience to the injected noise.

\textbf{Findings}. As shown in Figure ~\ref{fig:robust_in_appendix}, CaDDi maintains significantly lower perplexity under all noise injection conditions compared to D3PM and MDLM. This highlights the benefit of CaDDi's non-Markovian architecture, which conditions on the full generative trajectory rather than a single latent state. Notably, CaDDi is more effective at self-correction when noise is injected early (e.g., $t=T / 2$ ), where traditional Markovian models suffer the most from cascading inference errors.

These results validate the design motivation behind CaDDi and demonstrate its ability to recover from intermediate perturbations, making it particularly well-suited for scenarios with partial or uncertain context.

\newpage
\section{Generation Samples from \method{}}


\subsection{Amazon Polarity}

\begin{figure*}[h]
\centering
\begin{minipage}{0.9\textwidth}  
\hrule height 0.8pt
\vspace{0.5em}
\small

\textcolor{purple}{this is a negative review:}

\textbf{title} : were they afraid of newborns?

\textbf{content} : this book was terrible. it did not address the issues at all regarding newborns, which i assume are destined to inconsequential every so often. there was no postiveness in the middle of the book, and most of it was rushed to press. whoever wrote the previous review should have gotten an apologia as their final printer of the book. this writer did so only in the first few pages, it was as if she just taped pages and put one in her books and taped it - to - more than - in ...

\vspace{0.5em}
\hrule
\vspace{0.5em}

\textcolor{purple}{this is a negative review:}

\textbf{title} : blechhhhh content

\textbf{content} : wow, isn't it nice to see such juvenile drivel such toole so popular on the net. the editor in the back of the book notes : " to preserve reynolds, an imaginary london child having sex with another high - school girl ( p. s. ) you don't know what kind of busted ethan padar worried, or would mentor a house floor fort chemist who sullys ceilings with zero light. the whole novel is lame, poorly written and awfully mediocre. the " plot " is transparent from the first page ...

\vspace{0.5em}
\hrule
\vspace{0.5em}

\textcolor{purple}{this is a negative review:}

\textbf{title} : pile of 1 \% crap

\textbf{content} : this movie was disaster. i think the director wanted cheap hair and vaginal characters, scaring everyone on the back of the face of the audience with pointless bigs or even no, uncreative dialouge. the ruined 30 minutes most of this movie, rushes in and out of the movie. i'm a huge fan of slasher flicks, even makes huge claims to hard work on hair. the only funny part in this movie is when the killer begins stalking all the exaggerated body - climbing and challengers, and the two - dimensional ...

\vspace{0.5em}
\hrule
\vspace{0.5em}

this is a negative review:

\textbf{title} : \textcolor{purple}{what an insult to music lovers!}

\textbf{content} : the true diva that deserves the wide spread islam and the lies of the young musicians who make this music allows for such glittery displays and excuses. she never needs to spend our high time singing this " exquisite in its gem ". the artist is just a shadow of her most famous but she has no talent. where she excels is the terribile and exquisite expression of native american descent that will make dog earpin jams easy to count. she has a good voice but she hasn't mastered it. oh, she needs to go ...

\vspace{0.5em}
\hrule
\vspace{0.5em}

this is a negative review:

\textbf{title} : \textcolor{purple}{miserable album content}

\textbf{content} : i owned this album and when it came out i loved it. unlike the last albums her voice was able to make coheed and connie told a cool gentle sigh. i figured the first few tracks could have been strengthened by absolutely no more stale vocals, but they are crafted to off the page. this person writes her songs so she has more power and intonation than true instincts. i like to listen to the song that she was destined to write. get the single years that are runaways. " only excuse me on the radio " is sweetly sung by anyone, ...

\vspace{0.5em}
\hrule height 0.8pt
\end{minipage}
\caption{Sample generated using \method{}, conditioned on a negative prompt. In each sample, the \textcolor{purple}{purple} text indicates the prompt.}
\label{fig:amazon_result_negative}
\end{figure*}

\newpage
\makeatletter
\setlength{\@fptop}{0pt}   
\makeatother
\begin{figure*}[t]
\centering
\begin{minipage}{0.9\textwidth}  
\hrule height 0.8pt
\vspace{0.5em}
\small

\textcolor{purple}{this is a positive review :}

\textbf{title} : one for the head bob fans!

\textbf{content} : i've won stadium since the mid 70's. this double lp has all of the greatest hits on it! so why not split their support. this has soul in the covers ( not on the cover ). a hercules and a freddie and theatrics contribution is true, and paul petrie and leslie heralbull rock at the same time. even the stevie ray \& carole or even stevie ray drivel adds some new comparatively to the cole porter music. a truly unique collection of a performers work, and the cd succeeds ...

\vspace{0.5em}
\hrule
\vspace{0.5em}

\textcolor{purple}{this is a positive review :}

\textbf{title} : jk rowling is a sci - fi crime novelist!

\textbf{content} : in harry's hell - called scream, jk rowling gets back the bullet to get to jk's responsiblity. i liked this book so much. the storyline was great, the characters well drawn and great, and jk rowling made me care about what happened. this is the second book in the harry " hallie " tray fouls trilogy including a sister named calvin, who is young. lucy prince is young again following her sister and her intense ex - boyfriend michaelk bo ...

\vspace{0.5em}
\hrule
\vspace{0.5em}

\textcolor{purple}{this is a positive review : }

\textbf{title} : a great story and excellent special effects and overall worth the film

\textbf{content} : trust me, this is much better than the original. set up for heaven's gate and the framer though works well - it becomes predictable and starts going nowhere. the story becomes more developed later and really makes a good place for truman if you enter into the dsi world you're never supposed to forget. it's almost as engaging with this surprisingly heavy drama that's subtle why nothing happens after you have watched this one to get that your expectations and expectations sading down. i recommend definitely seeing this ...

\vspace{0.5em}
\hrule
\vspace{0.5em}

this is a positive review:

\textbf{title} : \textcolor{purple}{simply the best}

\textbf{content} : i have been wanting to have a safe with my son my son. we used to play w / his cardboard toy box, but thought this was just a hazard. well, after true terror in my baby's swing, i hunted down the safe and amazon had the best price i could find. great product that was a huge relief from our efforts. my 7 month old loves to play with this at least 10 different times that i put him in the swing. i can see how i will get this useful for the one he has. i am so glad i ...

\vspace{0.5em}
\hrule
\vspace{0.5em}

this is a positive review:

\textbf{title} : \textcolor{purple}{a surprise!}

\textbf{content} : this box is a prelude to the first golden age of horror crooning ( which correctly apes de croes as he tends to say a wayans juveniles below us ). and these are individual breakfast of murders. the writers of these crime / mysteries / new orleans exhibit that encompasses the catholic church in their b \& w documents. best blurb of all those who would read about the errors they witness and don't believe. an interesting twist ending to the loose ends. hercule poirot is cast as a smart cop undercover. the two big funny g ...

\vspace{0.5em}
\hrule height 0.8pt
\end{minipage}
\caption{Sample generated using \method{}, conditioned on a positive prompt. In each sample, the \textcolor{purple}{purple} text indicates the prompt.}
\label{fig:amazon_result_positive}
\end{figure*}


\end{document}